\documentclass{article}

\usepackage{microtype}
\usepackage{graphicx}
\usepackage{subfigure}
\usepackage{booktabs} %

\usepackage{hyperref}

\usepackage[accepted]{icml2024}

\usepackage{amsmath}
\usepackage{amssymb}
\usepackage{mathtools}
\usepackage{amsthm}

\usepackage{dsfont} %

\usepackage[capitalize,noabbrev]{cleveref}

\theoremstyle{plain}
\newtheorem{theorem}{Theorem}[section]

\newtheorem{corollary}[theorem]{Corollary}
\theoremstyle{definition}
\newtheorem{definition}[theorem]{Definition}
\newtheorem{assumption}[theorem]{Assumption}
\newtheorem{conjecture}[theorem]{Conjecture}
\theoremstyle{remark}

\usepackage[inline]{enumitem}
\usepackage{acronym}
\usepackage{balance}
\usepackage{tikz}
\usetikzlibrary {arrows.meta}

\acrodef{ML}{machine learning}
\acrodef{SUWR}{sequential unmasking without reversion}
\acrodef{MSE}{mean squared error}
\acrodef{CAE}{concrete autoencoder}

\newcommand{\mpara}[1]{\noindent{\bf #1}}

\newcommand\mydots{\makebox[0.8em][c]{.\hfil.\hfil.}}

\renewcommand\algorithmiccomment[1]{%
  \hfill\textit{\#\ #1}%
}

\icmltitlerunning{Local Feature Selection without Label or Feature Leakage for Interpretable Machine Learning Predictions}

\begin{document}

\twocolumn[
\icmltitle{Local Feature Selection without Label or Feature Leakage  for~Interpretable~Machine~Learning~Predictions}

\icmlsetsymbol{equal}{*}

\begin{icmlauthorlist}
\icmlauthor{Harrie Oosterhuis}{equal,radboud}
\icmlauthor{Lijun Lyu}{equal,delft}
\icmlauthor{Avishek Anand}{delft}
\end{icmlauthorlist}

\icmlaffiliation{radboud}{Radboud University, Nijmegen, The Netherlands}
\icmlaffiliation{delft}{TU Delft, Delft, The Netherlands}

\icmlcorrespondingauthor{Harrie Oosterhuis}{harrie.oosterhuis@ru.nl}

\icmlkeywords{Machine Learning, ICML}

\vskip 0.3in
]

\printAffiliationsAndNotice{\icmlEqualContribution}  %

\begin{abstract}
Local feature selection in machine learning provides instance-specific explanations by focusing on the most relevant features for each prediction, enhancing the interpretability of complex models. 
However, such methods tend to produce misleading explanations by encoding additional information in their selections.
In this work, we attribute the problem of misleading selections by formalizing the concepts of label and feature leakage. 
We rigorously derive the necessary and sufficient conditions under which we can guarantee no leakage, and show existing methods do not meet these conditions.
Furthermore, we propose the first local feature selection method that is proven to have no leakage called SUWR.
Our experimental results indicate that SUWR is less prone to overfitting and combines state-of-the-art predictive performance with high feature-selection sparsity.
Our generic and easily extendable formal approach provides a strong theoretical basis for future work on interpretability with reliable explanations.

\end{abstract}

\section{Introduction}

Feature attributions and feature selections in interpretable machine learning (ML) help users understand how much each input feature influences the output of the model~\cite{du2019techniques, molnar2020interpretable}.
One prominent family of methods are designed for local feature selection, a.k.a.\ instance-wise feature selection, for interpretable ML~\cite{rationale-survey}.
These approaches aim to only select the most-important features per instance and to exclude the rest during inference~\cite{li2017feature}, thereby making the predictions by the model easier to interpret.

Let $i$ refer to an instance in a dataset with $x_i \in \mathbb{R}^d$ as its $d$-dimensional feature vector representation and $y_i$ as its accompanying label to be predicted.
A feature selector $\zeta$ takes $x_i$ as input and outputs a feature mask $h_i \in \{0,1\}^d$, either through a stochastic or deterministic process: $h_i \sim \zeta(x_i)$.
Let $x_i \odot h_i$ indicate the masked features that results from applying $h_i$ to $x_i$, where all non-selected features are masked.
We denote a masked feature with $\emptyset$, to clearly differentiate it from a zero value, and the $j$th element in a vector with $[j]$:
\begin{equation}
    (x_i \odot h_i)[j]  \coloneqq
    \begin{cases}
        x_i[j] & \text{if } h_i[j] = 1, \\
        \emptyset & \text{if } h_i[j] = 0.
    \end{cases}
\end{equation}
here $\zeta$ is a local selector and  produces a different mask for each instance.
We note that local feature selection differs from global feature selection which reveals feature importance on the dataset level, as it has a fixed mask for all instances~\cite{cae,lemhadri2021lassonet, yamada2020feature,lee2021self}.
By being able to vary masks, local methods are more flexible and can give more in-depth insight into the importance of features in individual instances~\cite{invase,DBLP:conf/aaai/tabnet}.

A widely used setup for local feature selection is to follow a selector-predictor architecture that is typically jointly optimized~\cite{invase, DBLP:conf/aistats/explanation-encode-pred,DBLP:conf/aaai/tabnet}.
More precisely, let $f$ be the predictor model that can take masked features as input: $f(x \odot h)$, importantly, $h$ can be inferred exactly from $x \odot h$.
The optimization of the selector $\zeta$ and predictor $f$ is usually based on a linear combination of a prediction loss $L$ and the sparsity of a mask  $\lVert h \rVert$ to enforce high sparsity (and hence interpretability).
For a dataset of $N$ instances, we use:
\begin{equation}
    \mathcal{L}(\zeta,\! f) \!\!\coloneqq\!\! \frac{1}{N}\!\sum_{i=1}^N \mathbb{E}_{h_i \sim \zeta(x_i)}\!\big[
    L(f( x_i \odot h_i), y_i) + \lambda \lVert h_i \rVert \big],
    \label{eq:genericloss}
\end{equation}
where the parameter $\lambda \in \mathbb{R}_{>0}$ balances feature sparsity against predictive performance.
The loss thus incentivizes the exclusion of features that do not contribute to high predictive performance, consequently, the selector should learn only to select the features that are the most important for accurate predictions.

Whilst the reasoning behind the local feature selection approach appears intuitive, previous work has found a \textit{fundamental flaw}: local methods can choose features that provide high predictive performance but clearly are \emph{unfaithful} explanations of feature importance~\citep{interpret-social-attribution}.
\citet{DBLP:conf/aistats/explanation-encode-pred} discuss a selector and predictor combination for digit classification where a single pixel is selected per image, yet optimal accuracy prediction is maintained.
Instead of selecting features based on importance, their selector learned to encode a prediction of the digit in the selection mask $h$.
Because they are optimized jointly, the predictor also learned the relation between the encoding and the original prediction.
In other words, instead of selecting the most important features, the behavior of the selector was aimed at passing as much information about the corresponding label as possible.
The resulting selections thus provide misleading explanations that give false insights into the prediction process.
As a remedy, \citet{DBLP:conf/aistats/explanation-encode-pred} add noise to the selection mask $h$; whilst this appears to improve the situation, it does not address the underlying problem.
To the best of our knowledge, no existing local feature selection method can guarantee that their selections never  provide misleading explanations by encoding  additional information.

In this paper, we provide the first formal approach to the issue of additional information being encoded in local feature selections.
We name this problem \emph{leakage} and define it using two novel formal concepts: \emph{label leakage}, where information about the label is encoded in a local selection, and \emph{feature leakage}, where information about the values of non-selected features is encoded in a local selection.
Subsequently, we derive the sufficient and necessary properties of a local feature selection method, it appears no existing method meets these criteria.

To address this problem, we propose two methods for optimizing local feature selection policies that are guaranteed to have no leakage.
First, we introduce a novel linear programming method to search for the optimal selection and prediction policy for any desired sparsity and accuracy tradeoff.
This method is highly effective but can only be applied to problems with complete knowledge that are of small scale, which means it has limited practical utility.
Second, we introduce a novel method that is much more practical and widely applicable called \acfi{SUWR}.
\ac{SUWR} selects features over several sequential decision rounds, where each decision is based only on the values of features that were selected in previous rounds and decisions cannot be reversed in subsequent rounds.
We prove that it is impossible for \ac{SUWR} to encode information about non-selected features or any labels, since it never had access to those values when deciding what to select.
Moreover, we conjecture that when the feature distribution fully supports the Cartesian product of possible feature values, \ac{SUWR} is the only solution without leakage, because it captures all possible policies that have no leakage.
Our experimental results indicate that \ac{SUWR} is less prone to overfitting and combines state-of-the-art predictive performance with high feature-selection sparsity.
Furthermore, the sequential decisions of \ac{SUWR} provide a novel way to explain predictions by giving a narrative of how predictions are formed (e.g., Figure~\ref{fig:shoes}), a unique insight that previous methods do not provide.
The \ac{SUWR} method can be applied to various forms of data and types of model architectures and optimization, its approach is generic and easily extendable.

\subsection{Brief related work}

Approaches in interpretable machine learning have been categorized into \textit{explaining trained models in post-hoc manner} ~\citep{lime,saliency,deeplift,shap,fastshap-KL} and \textit{building intrinsically explainable models} \citep{l2x,invase,expred}. 
Local feature selection methods use only a few relevant features to generate each prediction and thus are popular for intrinsical explainability.  
These methods mainly adhere to a selector-predictor architecture, e.g., CAE~\citep{cae}, L2X~\citep{l2x}, INVASE~\citep{invase} and REAL-X~\citep{DBLP:conf/aistats/explanation-encode-pred}; or both are performed within a single model, e.g., TabNet~\citep{DBLP:conf/aaai/tabnet}.
The resulting feature selections are then supposed to serve as explanation of the corresponding predictions.
However, several recent works question this use of feature selections as explanations~\citep{interpret-social-attribution,irrationality-rationales}. 
Specifically, earlier work has found that the joint-training regime can result in high sparsity irrespective of the relevance of the selected features~\citep{DBLP:conf/aistats/explanation-encode-pred}. 
In this paper, we solve this fundamental discrepancy by providing necessary and sufficient conditions that a local model selection method should satisfy to provide faithful explanations.
See Appendix~\ref{appendix:relatedwork} for a more detailed discussion of related work.

\section{Leakage in Feature Selection}
\label{sec:leakage}

This section introduces a formal definition of leakage based on label and feature leakage.
Subsequently, we use them to prove the necessary and sufficient conditions for leakage.

To keep our terminology succinct, we define \emph{leakage} as either \emph{feature leakage} or \emph{label leakage}, thus:%
\footnote{Our definition is different but related to the concept of \emph{data leakage}: the availability of information during optimization that is unavailable during inference~\cite{kaufman2012leakage}.}
\begin{definition}
\label{def:leakage}
A feature selector does not have leakage, if it has neither label leakage (Definition~\ref{def:labelleakage}) nor feature leakage (Definition~\ref{def:featureleakage}).
\end{definition}
Table~\ref{tab:featleak} displays an intuitive example of leakage where a selection policy mask perfectly encodes all information about the label and non-selected features.

\subsection{Formalization of label leakage in feature selection}

Colloquially, we understand label leakage to be the problem where the selection mask $h$ encodes information about the label.
In the context of interpretable \ac{ML}, the purpose of $h$ is to select the features that provide the most salient information.
Therefore,  this purpose is entirely defeated by the injection of additional information about the label in $h$.
This problematic behavior has been observed in previous work~\citep{DBLP:conf/aistats/explanation-encode-pred, interpret-social-attribution}, however, to the best of our knowledge, no one has introduced a formal definition of this issue yet.

In our notation, we denote $s^\text{in}$ as the  set of indices of the selected features (\underline{in}cluded) and $s^\text{ex}$ for the non-selected features (\underline{ex}cluded).
To keep our notation brief, we define:
\begin{definition}
\label{def:omegaset}
$\Omega$ is the set of all possible selections of feature values and label values:
\begin{equation}
\begin{split}
\Omega \coloneqq \{&(x,y,s^\text{in}\!\!,s^\text{ex}) : p(x) > 0 \land p(s^\text{in}\!\!,s^\text{ex} \mid x, \zeta) > 0
 \\
& \land p(x[s^\text{in}],y) > 0 \land s^\text{in} \cup s^\text{ex} = \{1, 2, \mydots, d\}\}. 
\end{split}
\end{equation}
\end{definition}
Our proposed definition of label leakage is based on the idea that the selection $h$ should not be able to provide information about the label.
For a selection, $(x,y,s^\text{in},s^\text{ex})\in\Omega$, the predictive information in this selection can be represented by the \emph{natural} label distribution conditioned on the selected feature values: $p(y \mid x[s^\text{in}])$.
This distribution can be further conditioned the fact that $s^\text{in}$ has been selected by the selector $\zeta$: $p(y \mid x[s^\text{in}], h[s^\text{in}] = 1, h[s^\text{ex}] = 0, \zeta)$.
The key insight in our definition is that when there is no label leakage, these distributions should be equal.
\begin{definition}
\label{def:labelleakage}
A feature selector $\zeta$ does not have label leakage,
if conditioning the label distribution on the selection of features by $\zeta$ does not change the label distribution:
\begin{align}
&\forall (x, y, s^\text{in},s^\text{ex}) \in \Omega,
 \label{eq:labelleakage} \\ & \nonumber\qquad
p(y \mid x[s^\text{in}]) = p(y \mid x[s^\text{in}], h[s^\text{in}] = 1, h[s^\text{ex}] = 0, \zeta).
\end{align}
\end{definition}
In other words, if the knowledge that a feature selection comes from a specific selector $\zeta$ changes the probability of a label, then $\zeta$ has label leakage.
Imagine two masked feature values: $x_1 \odot h_1 \!=\! x_2  \odot h_2$, one made with a uniform random selection, the other with $\zeta$, if predictions are only based on the selected feature values then both should lead to the exact same predictions: $\forall y,\, p(y \mid x_1 \odot h_1) = p(y \mid x_2 \odot h_2, \zeta)$.

\begin{table}[t]
\centering %
\caption{Example of feature and label leakage in feature selection (non-selected features are omitted). %
The label $y$ is the sum of the two independent features, therefore, perfect label prediction should only be possible with both features.
However, each $x\odot h$ value is matched with a single label and set of feature values,
thereby, this solution provides 100\% accuracy in label prediction and feature reconstruction, with a $62.5\%$ feature reduction.
This combination of performance and sparsity is only possible because of leakage.
}
\label{tab:featleak}
 \resizebox{\columnwidth}{!}{

\begin{tabular}{c | c c  | c c  | c c  | c}
\toprule
\!\!$p(x, y, h )$\!\! &
\!$x[1]$\!\! & \!\!$x[2]$\!  &
\!$h[1]$\!\! &  \!\!$h[2]$\!  &
\hspace{-1mm}$(x\odot h)$[1]\hspace{-1.5mm} & 
\hspace{-1.5mm}$(x \odot h)$[2]\hspace{-1mm} &  
y \\
\midrule
   0.25 & 1 & 1 & 1 & 0 & 1 &  & 2 \\
   0.25 & 0 & 1 & 0 & 1 &  & 1 & 1 \\
   0.25 & 1 & 0 & 0 & 1 &  & 0 & 1 \\
   0.25 & 0 & 0 & 0 & 0 &  &  & 0 \\
   \bottomrule
\end{tabular}

 }
 \vspace{-\baselineskip}
\end{table}

\subsection{Formalizing feature leakage in feature selection}
Analogous to label leakage, we also propose the concept of feature leakage where the selection mask $h$ encodes information about non-selected features.
As illustrated in Table~\ref{tab:featleak},
we motivate the prevention of feature leakage with two arguments:
 \begin{enumerate*}[label=(\roman*)]
\item feature leakage defeats the purpose of feature selection as information about the values of non-selected features is not actually excluded; and
\item when there is a correlation between features and labels, a basic assumption in machine learning~\cite{bishop2006pattern}, feature leakage implies label leakage.
\end{enumerate*}
Therefore, it also seems infeasible to prevent label leakage without also tackling feature leakage.
We formally define feature leakage as:
\begin{definition}
\label{def:featureleakage}
A feature selector $\zeta$ does not have feature leakage,
if conditioning the feature distribution on the selection of features by $\zeta$ does not change the feature distribution:
\begin{equation}
\begin{split}
    & \forall (x, y, s^\text{in},s^\text{ex}) \in \Omega, \quad
     p(x[s^\text{ex}] \mid x[s^\text{in}])
     \\ &\qquad\quad
     = p(x[s^\text{ex}] \mid x[s^\text{in}], h[s^\text{in}] = 1, h[s^\text{ex}] = 0, \zeta).
\end{split}
\label{eq:featureleakage}
\end{equation}
\end{definition}
Similar to label leakage, the intuition behind feature leakage is that knowing that a feature selection was made by $\zeta$ should not affect the probability of non-selected feature values.

\subsection{The necessary and sufficient conditions for leakage}
\label{sec:conditions}

From these formal definitions of feature leakage and label leakage, we derive the sufficient and necessary conditions for a feature selector without leakage in Appendix~\ref{section:proofconditions}.
We find the following:
\begin{corollary}
\label{corollary:mainresult}
A feature selector does not have leakage if and only if every probability for every possible feature selection does not depend on any label values or any non-selected feature values:
\begin{align}
     & \forall (x, y, s^\text{in}\!, s^\text{ex}) \in \Omega, \quad
     p( h[s^\text{in}] = 1, h[s^\text{ex}] = 0 \mid x[s^\text{in}], \zeta)
     \nonumber \\ & \qquad\;
      = p( h[s^\text{in}] = 1, h[s^\text{ex}] = 0 \mid x[s^\text{in}], x[s^\text{ex}], y, \zeta).
\end{align}
\end{corollary}
\begin{proof}
    Follows directly from Theorem~\ref{theorem:labelleakage} and Theorem~\ref{theorem:featleakage} in Appendix~\ref{section:proofconditions}. 
\end{proof}
In other words, a feature selector has no leakage if the probability of a selection is only determined by the values of the selected features, and not by the label or non-selected feature values.
Therefore, for any possible feature values $x$ and any label value $y$ and any selection mask $h$, any change in the label or in any of the features not selected by $h$ should not result in a different probability for the selection: $\zeta(h \mid x)$.
Thus, for any possible feature values $x'$ and label value $y'$, where the selected features have identical values: $x \odot h = x' \odot h$, the probability of the selection should be identical: $\zeta(h \mid x) = \zeta(h \mid x')$.

Intuitively, we can understand that if the value of the label or unselected features changes the behavior of $\zeta$, then it could be possible to infer information about unselected features or the label from the behavior of $\zeta$.
Accordingly, we can prove a feature selector has leakage by finding a single example of two pairs of $(x,y)$ and $(x',y')$ for which the above condition does not hold.
Conversely, to prove a feature selector has no leakage, we have to rule out the possibility of such an example entirely.

\section{A Linear Programming Solution}
\label{sec:linprog}

We now propose our first method that meets the above criteria using linear programming~\cite{dantzig1963linear}.
It requires full knowledge of the problem setting, i.e., $p(x,y)$ is known completely, and assumes a finite set of possible values for $x$.
In this setting, the perfect predictor is available, e.g., for a mean squared error loss:
$f^*( x \odot h) = \mathbb{E}_x[y \,|\, x\! \odot\! h] = \sum_{x' : x'\odot h = x \odot h} p(x') \sum_{y} p(y \,|\, x') y$,
and thus, only $\zeta$ has to be optimized.
Corollary~\ref{corollary:mainresult} shows that the probability of any masked feature vector $x \odot h$ should only depend on the selected features, since:
\begin{align}
\forall (x, x', h), \; &\big( p(x) > 0 \land p(x') > 0 \land (x \odot h) = (x' \odot h) \big)
\nonumber \\
& \hspace{1.75cm} \longrightarrow \zeta(h \mid x) = \zeta(h \mid x').
\label{eq:linprog1}
\end{align}
Therefore, for optimization, we only have to consider a single probability variable for every possible set of values for $x \odot h$.
The probability variables should be chosen to minimize: $\mathcal{L}(\zeta, f^*)$ (Eq.~\ref{eq:genericloss}), under the constraint that they describe valid probability distributions:
\begin{align}
&\forall x, \; p(x) > 0
\longrightarrow \Big(
\sum\nolimits_{h \in \zeta(x)}
\zeta( h \mid x ) 
\label{eq:linprog2} \\[-1ex] \nonumber & \hspace{0.9cm}
=
\sum_{s^\text{in}\!, s^\text{ex} : s^\text{in} \cup  s^\text{ex} =  \{1, 2, \mydots, d\}\hspace{-2cm} }
p( h[s^\text{in}] = 1, h[s^\text{ex}] = 0 \mid x[s^\text{in}], \zeta ) = 1
\Big).
\end{align}
Appendix~\ref{appendix:linearprogram} details how this task is translated to a linear programming problem.
Whilst its requirements limit it to unrealistic toy problems, this method enables us to closely approximate the Pareto optimal front of selection without leakage, which we use in our analysis of existing methods.

\section{Sequential Unmasking without Reversion}
\label{sec:method}

In this section, we propose a more practical method titled \acfi{SUWR}, which describes a feature selection algorithm that provenly has no leakage, but is applicable to more realistic settings than the linear programming solution.
\ac{SUWR} guarantees no leakage by approaching the selection of features as a sequential decision process where each decision is only based on a specific subset of feature values, and no decision can be reversed at a later step.
The core of \ac{SUWR} is its selection inference algorithm, which is agnostic to what underlying \ac{ML} model is used and how it is optimized.
Therefore, \ac{SUWR} can be seen as a generic framework that can easily be extended and adapted to specific feature selection problems.

\subsection{Feature selection inference with SUWR}
\label{sec:method:inference}

From Section~\ref{sec:leakage}, we know that a feature selector $\zeta$ without leakage, should base the probability of a specific selection only on the values of the selected features.
As discussed in Section~\ref{sec:linprog}, the probability distribution over each possible selection of feature values has to be valid.%
\footnote{
Meeting both of these criteria is not trivial, since a standard normalization term would depend on all possible selections for an instance $x$ and thus also on non-selected features;
i.e., $\zeta(h \mid x) \coloneqq \hat{\zeta}(h \mid x) / \sum_{h} \hat{\zeta}(h \mid x)$ is not allowed since the normalizing denominator depends on all feature values.
}
Based on these properties, we propose \ac{SUWR} which meets these criteria through sequential selection.
Algorithm~\ref{alg:inference} describes inference with \ac{SUWR} in pseudocode,
the remainder of this section describes it step-by-step.

\begin{algorithm}[t]
\caption{Inference with the \acs{SUWR} method.} 
\label{alg:inference}
\begin{algorithmic}[1]
\STATE \textbf{Input}: Features: $x$, Max-$t$: $T$, Selector: $\zeta$, Predictor: $f$
\STATE $h \leftarrow \mathbf{0}$ %
\FOR{$t \in [0,1,\ldots,T-1]$}
    \IF{Bernoulli\_Trial$( \zeta^{t}_\text{stop}(x \odot h))$}
        \STATE \textbf{Return}: $(f(x \odot h), h)$ %
    \ENDIF
    \STATE $h \leftarrow h + \text{Sample\_Mask}(\zeta^{t}_\text{select}(x \odot h))$  \COMMENT{Eq.~\ref{eq:SUWR:sampling}}
\ENDFOR
\STATE \textbf{Return}: $(f(x \odot h), h)$ %
\end{algorithmic}
\end{algorithm}

\ac{SUWR} requires a model $\zeta$ that can output a stop probability and a distribution to sample feature indices, given an input of masked features.
The feature selection process takes place over $T$ steps, each step starts by deciding whether to stop the process, and if not, which features to select next.
For a step $t$, where $0 \leq t < T$, a Bernoulli trial is performed according to $\zeta^t_\text{stop}(x \odot h^{t})$ and if successful then the process is stopped and $h^t$ is the final feature selection and $f(x \odot h^t)$ the final prediction.
Otherwise, the process continues and a new set of feature indices is sampled and added to the selection mask:
\begin{equation}
u^{t} \sim \zeta^t_\text{select}(x \odot h^{t}), \qquad h^{t+1} = h^{t} + u^{t}.
\label{eq:SUWR:sampling}
\end{equation}
Importantly, both the stop probability and the sampling of new features are only conditioned on the values of features selected in the previous steps ($x \odot h^{t}$).
Accordingly, the first step ($t=0$) starts with an empty mask $h^0 = \mathbf{0}$, and the stop probability $\zeta^0_\text{stop}(x \odot h^0) = \zeta^0_\text{stop}(\emptyset)$ is constant over $x$, similarly, the feature distribution $\zeta^0_\text{select}(x \odot h^{0})$ is the same for every $x$ in the first step.
Additionally, since each step only adds features to the selection and never removes any, the probability of the decisions that lead to $h^t$ in a step $t$, only depends on the values of features selected in previous steps ($ x \odot h^{t-1}$).
If the final step $t = T-1$ is reached, then the process is automatically stopped ($\zeta^T_\text{stop}(\cdot) = 1$) and the final selection is $h^T$ and the final prediction $f(x \odot h^T)$.

As we can see, \ac{SUWR} is completely agnostic to what the underlying model $\zeta$ and predictor $f$ are; it only requires them to handle masked inputs and $\zeta$ to output a stop probability and feature distribution.
The parameter $T$ acts as a computational budget as it ensures the process halts within $T$ steps.
Additionally, $T$ is also a feature budget when $\zeta$ limits the number of features to be sampled per step.

Appendix~\ref{appendix:methodproof} provides a full proof that proves \ac{SUWR} has no leakage.
The intuition behind this property is straightforward:
Any decision to select a feature is never based on information from (thus far) unselected features.
Therefore, the value of a feature that is not in the final selection could never affect its probability.
Furthermore, the process guarantees a selection is always made, thereby providing a valid probability distribution over all possible feature selections.

In addition, in Appendix~\ref{appendix:onlymethodproof} we conjecture that the \ac{SUWR} algorithm describes every possible selection policy without leakage,
when the feature value distribution provides support for the Cartesian product of possible feature values:
\begin{equation}
\begin{split}
\forall i,j,a,b, \;\; &\big( p(x[i] = a) > 0 \land p(x[j] = b) > 0 \big)
\\&\qquad\quad
\longrightarrow p(x[i] = a, x[j] = b) > 0.
\end{split}
\label{eq:featureproduct}
\end{equation}
In other words, we conjecture that when Eq.~\ref{eq:featureproduct} holds, the inference of any feature selection policy without leakage can be computed by the \ac{SUWR} algorithm.
Therefore, in this setting, \ac{SUWR} captures \emph{all} solutions to feature selection without leakage,
and thus, \ac{SUWR} provides the \emph{only} solution to feature selection without leakage when Eq.~\ref{eq:featureproduct} is true.

\subsection{Optimization of \ac{SUWR} feature selection policies}
\label{sec:method:optimization}

While \ac{SUWR} inference strictly follows Algorithm~\ref{alg:inference} to prevent leakage,
there are no restrictions on the optimization of the underlying $\zeta$ and $f$ models.
Therefore, any optimization method can be chosen without risking the introduction of leakage.
For this paper, we propose a reinforcement learning optimization approach that is evaluated in our experiments.

The set of possible feature selections grows exponentially with the number of features, it is therefore important that we avoid iterating over all possibilities.
We use a REINFORCE approach~\cite{sutton1999policy} and repeatedly sample a set of $T$ selection steps while ignoring the stop probabilities.
Thus, we start at $t=0$ with the zero selection:
$\bar{h}^0_i = \mathbf{0}$,
and for each subsequent step $t$, we follow the SUWR procedure:
$\bar{u}^t_i \sim \zeta(x_i \odot \bar{h}^{t-1}_i )$,
$\bar{h}^t_i = \bar{h}^{t-1}_i + \bar{u}^t_i$.
For each datapoint $x_i$, this results in a sampled sequence of $T$ selection masks:
$\bar{H}_i = \{\bar{h}_i^0,\bar{h}_i^1,\!\mydots,\bar{h}_i^T\}$.
The probability that \ac{SUWR} stops at any $t$, conditioned on the sampled sequence is:
\begin{equation}
p_\text{stop}(t \mid \bar{H}_i)
\coloneqq
\zeta^t_\text{stop}(x_i \odot h^{t}_i) \prod_{j=0}^{t-1}\big(1- \zeta^j_\text{stop}(x^j_i \odot h^j_{i})\big).
\end{equation}
Using this formulation, we can create the following unbiased estimate of our generic loss function (Eq.~\ref{eq:genericloss}):
\begin{align}
    &
   \bar{\mathcal{L}}(\zeta,  f) \coloneqq
    \\&\;\;\;\;\;\; \nonumber
     \frac{1}{N} \sum_{i=1}^N  
    \sum_{t=0}^T  p_\text{stop}(t \mid \bar{H}_i)
    \big(L(f(x_i \odot \bar{h}^t_i), y_i) + \lambda \lVert \bar{h}_i^t \rVert \big).
\end{align}
Computing its gradient w.r.t.\ $\zeta_\text{stop}$ is straightforward; for the gradient w.r.t.\ $\zeta_\text{select}$, we use the log-trick from the REINFORCE method~\cite{sutton1999policy}.
Then, we apply standard gradient descent to optimize both $\zeta$ and $f$ based on our sampled loss $\bar{\mathcal{L}}$.

\subsection{Discussion}
\label{sec:method:discussion}

Since we can prove \ac{SUWR} has no leakage, each mask $h$ is guaranteed to indicate the only features that were used to make its corresponding prediction.
To the best of our knowledge, \ac{SUWR} is the first method to have this guarantee, therefore, we argue it is also the first feature selection method that guarantees its explanations are \emph{faithful}~\citep{interpret-social-attribution}.
Furthermore, the sequential selection procedure can be interpreted as a step-by-step narrative of how the prediction was constructed.
For example, Figure~\ref{fig:shoes} displays multiple steps of \ac{SUWR} on images of a sandal and a boot.
At each step, we can see what information became available to the predictor and how this changes its predictions.
Thereby, this step-by-step explanation provides even more insight than the final selection mask.
We believe \ac{SUWR} is the first approach that produces narrative explanations about feature importance.

While the guarantee of no leakage is a great advantage over existing methods, the \ac{SUWR} algorithm could potentially require more computational costs than previous approaches.
Namely, for each intermediate feature selection step, a call to $\zeta_\text{select}$ is made.
This could pose a challenge to data with high dimensionality, e.g., if $\zeta$ only selects a single feature per step, and thus a high $T$ should be chosen.
Luckily, the \ac{SUWR} framework is highly flexible and can be adapted to handle such situations better.
For instance, one can choose $\zeta_\text{select}$ to be a lightweight model that can choose multiple features at once.
In our experiments in Sections~\ref{sec:experiments-toy}~\&~\ref{sec:experiments-synthetic}, we choose $\zeta_\text{select}$ to be a model that selects one feature per step $t$; in contrast, for the experiment based on image data in Section~\ref{sec:experiments-MNIST}, we use a $\zeta_\text{select}$ that selects a patch of nine pixels per step.
This makes the resulting selection easier to interpret than one where individual pixels can be selected, while at the same time reducing the number of steps needed to select a complete image.
We expect that specific $\zeta_\text{select}$ models can be developed to increase the computational efficiency and scalability of \ac{SUWR} further.

Nevertheless, we want to note that there are some unintuitive aspects of \ac{SUWR} that seem to be unavoidable consequences from the definition of leakage.
In particular, at the first step ($t=0$) \ac{SUWR} selects features without conditioning on any feature values, thus this first step can be seen as a \emph{blind} selection.
While $\zeta_\text{select}(\emptyset)$ can be optimized to select the most informative features, its distribution over features must be the same for all possible values of $x$.
At first glance this may seem counter-intuitive, however, it appears that this is an inevitable consequence of selecting without leakage.
Consider a setting where we wish to select a single feature per $x$ without leakage, according to Corollary~\ref{corollary:mainresult}, the selection of a single feature can only depend on the value of that single feature.
However, if the distribution of features supports the Cartesian product of possible feature values (Eq.~\ref{eq:featureproduct}), then the probability of each mask is not dependent on any feature values.
To put this formally, let $h_{\text{only }i}$ indicate the mask where only feature $i$ is selected: $h_{\text{only }i}[i] = 1, \forall j \not= i, h_{\text{only }i}[j] = 0$, if only such masks can be chosen then the probability each mask is independent of any feature value since:
\begin{align}
&\zeta(h_{\text{only }i} \mid x[i])
\\
&\;\;
= 1 - \max_{\{x' : p(x') > 0 \land x[i] = x'[i]\}}\sum_{j: 0 < j < d \land i \not= j \hspace{-8mm}}  \zeta(h_{\text{only }j} \mid x'[j])
\nonumber\\
&\;\;
= 1 - \max_{\{x' : p(x') > 0\}}\sum_{j: 0 < j < d \land i \not= j \hspace{-8mm}}  \zeta(h_{\text{only }j} \mid x'[j])
= \zeta(h_{\text{only }i}),
\nonumber
\end{align}
where we rely on the fact that Eq.~\ref{eq:featureproduct} implies that the maximum operator over unselected feature values is a constant w.r.t.\ the value of any selected feature $x[i]$.
Thus, this derivation proves that, in this setting, blind selection is necessary for feature selection without leakage.

\begin{figure}[t]
    \centering
    \includegraphics[scale=0.4]{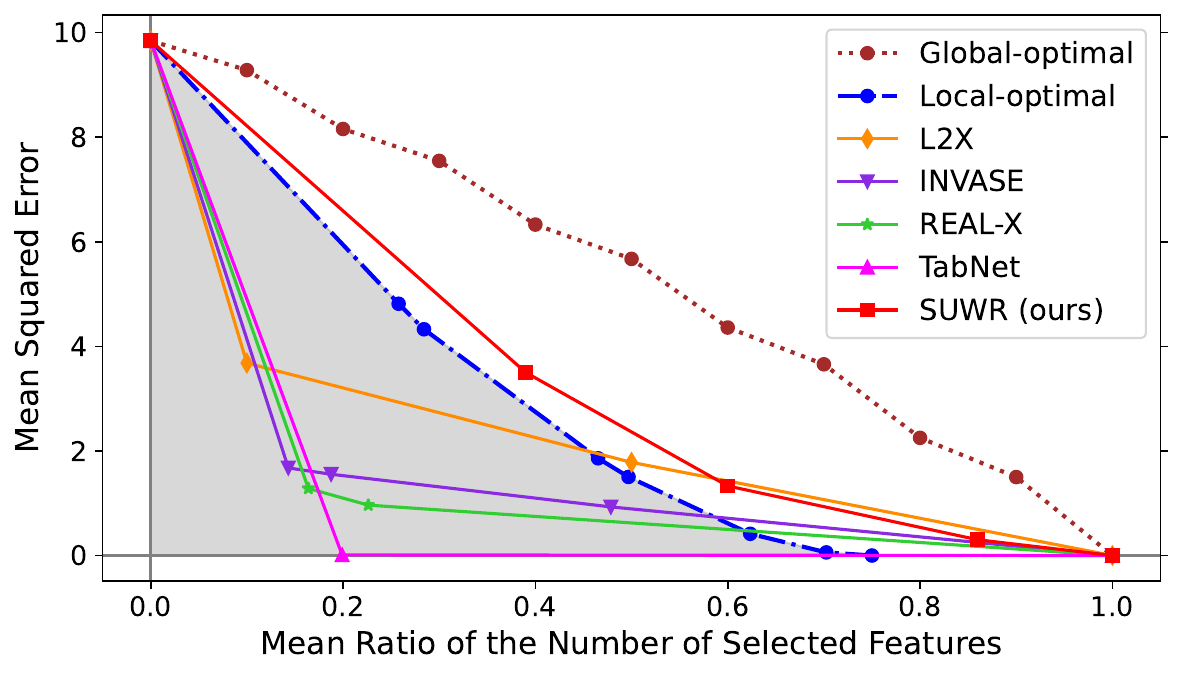}
    \vspace{-\baselineskip}
    \caption{
    Performance curves of the first experiment.
    Grey area indicates performance that is impossible without leakage.
    }
    \label{fig:toy-pareto}
    \vspace{-1.3\baselineskip}
\end{figure}

\section{Experiment 1: Pareto Front Analysis}
\label{sec:experiments-toy}

\mpara{Setup}.
Our first experiment is designed to identify whether existing methods have leakage.
For that, we design an idealized setup where complete information is available so the Pareto front can be approximated.
Leakage can then be identified by performance that exceeds that front.
Specifically, we design a toy problem with ten binary features: $x \in \{0, 1\}^{10}$, in a uniform distribution: $p(x) \coloneqq 1024^{-1}$.
As labels we use a sum of the product of feature pairs:
$y \coloneqq (\sum_{i = 1}^{5} x_{2 i -1} x_{2 i})^2$,
this induces feature redundancies that enable interesting local feature selection.
For example, if $x_1 = 0$ then the value of $x_2$ is irrelevant to $y$, but if $x_1 = 1$ then $x_2$ is relevant; this is a typical kind of pattern that only local feature selection methods can capture.
For this experiment, all methods are provided all possible values of $x$ and $y$, this creates a fair comparison with the Pareto front which is constructed using the same complete information.

\mpara{Methods}.
The comparison includes the following state-of-the-art methods:
\begin{enumerate*}[label=(\roman*)]
\item L2X~\citep{l2x};
\item INVASE~\citep{invase};
\item TabNet~\citep{DBLP:conf/aaai/tabnet}; 
\item REAL-X~\citep{DBLP:conf/aistats/explanation-encode-pred},
and
\item our proposed \ac{SUWR} method.
In addition, we added the following for further insight:
\item local-optimal, a close approximation of the Pareto front using the linear programming method from Section~\ref{sec:linprog};
and \item global-optimal, the Pareto front of global feature selection computed through brute-force.
\end{enumerate*}
All methods optimize the same feed-forward network architectures for $\zeta$, except TabNet which requires a method-specific architecture.
We repeat optimization with various $\lambda$ weights to visualize the tradeoff between feature sparsity and accuracy for each method.
Further details on the experimental setup can be found in Appendix~\ref{appendix:toy-dataset}.

\mpara{Results}.
Figure~\ref{fig:toy-pareto} displays the performance curves of each method in terms of \ac{MSE} and mean ratio of the number of selected features.
There is a large gap between the Pareto fronts of local and global feature selection, showing the usefulness of local selection in this setting.
However, \emph{all} of the baseline methods produce policies that improve over the Pareto front, which is \emph{impossible without leakage}.
For instance, TabNet only needs two features to achieve perfect prediction.
Clearly, given the formula for $y$, this is impossible with predictors that truly only use two features.
Therefore, our results prove that L2X, INVASE, TabNet and REAL-X \emph{all have leakage}, and thus, that their selectors encode additional information into their selections.
Even though REAL-X was specifically proposed to mitigate this issue by adding noise to $h$, our results prove that this strategy is not enough to prevent leakage.
In contrast, \ac{SUWR} is the only method that is close to the Pareto front and stays in the range of possible performance.
As expected, because \ac{SUWR} is guaranteed to have no leakage.

\setlength{\tabcolsep}{2.5pt}
\begin{table*}[t]
 \vspace{-0.6\baselineskip}
    \caption{Selection and prediction performance on the synthetic benchmark of the second experiment. Results are averages over five runs.}
    \vspace{-\baselineskip}
    \label{tab:syn-results}
    \begin{center}
    \resizebox{\textwidth}{!}{
    \begin{tabular}{l|ccc|ccc|ccc|cccc|cccc|cccc}
    \toprule
    Dataset & \multicolumn{3}{c|}{\bf Syn1 ($g_1$)}  &\multicolumn{3}{c|}{\bf Syn2 ($g_2$)} &\multicolumn{3}{c|}{\bf Syn3 ($g_3$)} &\multicolumn{4}{c|}{\bf Syn4 ($g_4$)} & \multicolumn{4}{c|}{\bf Syn5 ($g_5$)} & \multicolumn{4}{c}{\bf Syn6 ($g_6$)}\\
    \midrule
    Metrics & TPR$\uparrow$ & FDR$\downarrow$ & AUROC$\uparrow$ & TPR & FDR & AUROC & TPR & FDR & AUROC & CFSR$\uparrow$ & TPR$\uparrow$ & FDR$\downarrow$ & AUROC$\uparrow$ & CFSR & TPR & FDR & AUROC & CFSR & TPR & FDR & AUROC \\
    \midrule
    w/o FS & 100. & 82. & .578 & 100. & 64. & .789 & 100. & 64. & .854 & 100. & 100. & 64. & .558 & 100. & 100. & 64. & .662 & 100. & 100. & 55. & .692 \\
    Oracle & 100. & 0. & .700 & 100. & 0. & .895 & 100. & 0. & .903 & 100. & 100. & 0. & .818 & 100. & 100. & 0. & .823 & 100. & 100. & 0. & .902  \\
    \midrule
    L2X  & 33.2 & 33.6 & .675 & 44.6 & 55.4 & .872 & 66.0 & 34.  & .889 & 56.5 & 79.2 & 34.7 & .781 & 51.0 &  71.9 & 43.6 & .788 & 34.0 & 80.1 & 19.9  & .876 \\
    INVASE & 100. & 0. &  .692 & 100. & 0. & .873 & 95.0 & 0. & .883 & 56. & 91. & 10.2 & .792 & 40.7 & 76. & 2.2 & .780 &  60.7 & 89.4 & 7.0 & .877\\
    TabNet & 86.4& 57.9 & .667   & 98.7 & 5.6 & .885  & 96.6 & 9.7 &.903  & 99.7 & 91.5 & 29.5 & .789  & 98.9 & 92.5 & 36.2 & .791 & 100. & 97.5 &  23.6 & .870  \\
    REAL-X & 100. & 24.2 & .661 & 100. & 20.0 & .794 & 100. & 7.94 & .873 & 100. & 99.9&  41.9 & .748 & 100. & 99.8 & 52.4 & .774 & 100. & 97.2 & 8.27 & .842 \\
    \midrule
    \ac{SUWR} & 100. & 2.35 & .700 & 97.0 & 0. & .895 &  100.& 0. & .903 & 100. & 98.0 & 20.0 & .810 & 100.& 99.6 & 20.0 & .816 & 100. & 97.4 & 0.37 & .896\\
    \bottomrule
    \end{tabular}}
    \end{center}
    \vspace{-1.5\baselineskip}
\end{table*}

\mpara{Conclusion}.
These results conclusively prove that all of the baseline methods have leakage.
To the best of our knowledge, we can therefore conclude that \ac{SUWR} is the \emph{first} and the \emph{only} local feature selection method without leakage.

\section{Experiment 2: Synthetic Benchmark}
\label{sec:experiments-synthetic}
\mpara{Setup}.
Whilst \ac{SUWR} has excellent performance for the first experiment (Section~\ref{sec:experiments-toy}), it concerned an idealized complete-information setting.
Our second experiment aims to evaluate its generalizability by considering a more realistic setup where the training and test sets are separated.
For a better comparison with previous work, we use an existing benchmark~\citep{l2x,invase,DBLP:conf/aistats/explanation-encode-pred}.
In this setup, eleven features, $x \in \mathbb{R}^{11}$, are sampled from a normal distribution: $x[i] \sim \mathcal{N}(0, 1)$.
Labels are binary, $y \in \{0,1\}$, and sampled according to $p(y = 1 \mid x) \coloneqq \frac{1}{1 + g(x)}$.
The $g(x)$ function thus determines the relation between $x$ and $y$. 
Six different $g(x)$ functions are used, the first three use non-overlapping sets of features:
$g_1(x) \coloneqq \exp(x[1] x[2])$;
$g_2(x)  \coloneqq \exp(\sum_{i=3}^6 x[i]^2 - 4)$; and
$g_3(x) \coloneqq  -10\sin(2x[7]) + 2|x[8]| + x[9] + \exp(-x[10])$.
The latter three use a selection function based on the eleventh feature:
$z(x, g, g') \coloneqq  \mathds{1}[x[11] < 0] g(x)  +  \mathds{1}[x[11] \geq 0] g'(x)$,
to choose between the first three functions:
$g_4(x) \coloneqq z(x, g_1, g_2)$;
$g_5(x) \coloneqq z(x, g_1, g_3)$; and
$g_6(x) \coloneqq z(x, g_2, g_3)$.
Thereby, the latter are specifically designed for local feature selection where the eleventh feature (called the \textit{control-flow} feature) determines the relevance of the other features.
We use 10,000 independent samples for training and another 10,000 as the test set.

\mpara{Methods}.
The same methods are included as in the first experiment (Section~\ref{sec:experiments-toy}).
Additionally, we also train a predictor without feature selection (w/o FS) and another with an oracle selector that only selects the features used by $g(x)$ for each $x$. 
See Appendix~\ref{appendix:synthetic-dataset} for more details.

\mpara{Metrics}.
We use the same metrics as \citet{DBLP:conf/aistats/explanation-encode-pred}:
the \textit{true positive rate}: $\text{TPR} = \frac{\text{\# selected relevant features}}{\text{\# relevant features}}$;
the \textit{false discovery rate} $\text{FDR} = \frac{\text{\# selected irrelevant features}}{\text{\# selected features}}$;
and the \textit{control-flow selection rate} (CFSR): the frequency of selecting the eleventh feature.
To measure predictive performance, we use the \textit{area under the receiver operating characteristic curve} (AUROC).
We note that a low CFSR score indicates leakage especially when TPR or AUROC is high, because it means the feature selection method actually uses the control-flow feature but does not select it.

\mpara{Results}.
Table~\ref{tab:syn-results} displays our results on the synthetic benchmark test set.
Interestingly, there is a large gap in the AUROC between the baseline without feature selection and the oracle baseline in all settings,
this indicates that excluding irrelevant features can make prediction substantially easier. %

In terms of AUROC, \ac{SUWR} consistently has the highest performance of all methods (excluding the oracle), with especially high margins on the latter three settings (Syn4-6).
In the first three settings (Syn1-3), \ac{SUWR} reaches oracle performance; whilst among the other methods, only TabNet is able to reach oracle performance in the third setting (Syn3).
This is surprising, since the first experiments showed that the existing methods could reach extremely high performance through leakage.
However, a key difference with the first experiment is that  in this setting evaluation is based on a held-out test set.
Therefore, leakage could instead result in heavy overfitting in this setting, whereas it could not in the first experiment.
We believe that this explains why \ac{SUWR} has substantially higher predictive performance for the second experiment:
There are many more ways to overfit \emph{with} leakage than \emph{without}, as a result, \ac{SUWR} is less prone to overfitting than the existing methods.

In terms of correct feature selection, \ac{SUWR} has a near-perfect TPR that is greater than 97\% across all settings and a perfect CFSR of 100\% in the relevant settings (Syn4-6).
REAL-X is the only baseline that has comparable TPR and CFSR across all settings.
The FDR of \ac{SUWR} is consistently lower than all baselines in all settings, except for INVASE which does better in the fourth and fifth setting (Syn4-5).
Nevertheless, INVASE also has a very low CFSR and TPR in these settings, which strongly suggests that it is benefitting from leakage.
Accordingly, the possibility of feature leakage makes it difficult to compare the feature sparsity of \ac{SUWR} with the baselines.
Nonetheless, in our results, \ac{SUWR} has near-perfect TPR and perfect CFSR, and the best FDR of baselines with comparable TPR and CFSR.

Additional results in Appendix~\ref{appendix:synthetic-dataset} also show that \ac{SUWR} consistently learns to select the control-flow feature first and that \ac{SUWR} is very robust to the budget parameter $T$.

\begin{figure*}[t]
    \begin{center}
    \vspace{-0.25\baselineskip}
        \includegraphics[width=0.92\textwidth]{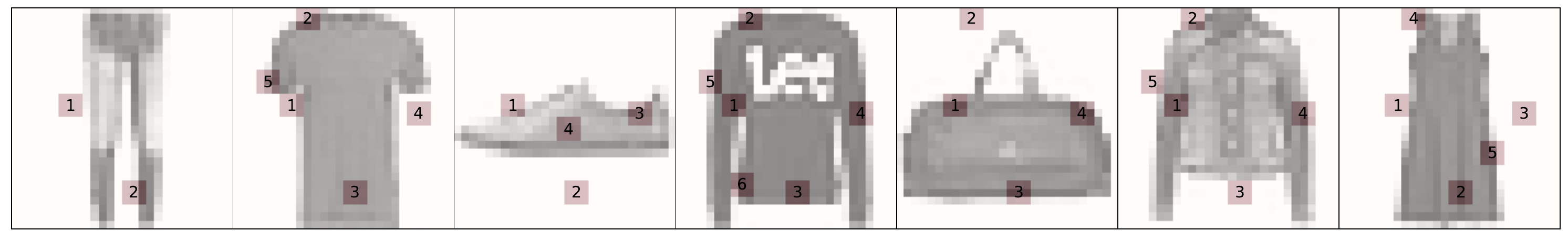}
    \vspace{-\baselineskip}
        \caption{
        Several selection masks produced by SUWR for different fashion items from fashion-MNIST. 
        Red squares indicate selected patches, the numbers shown inside indicate at what step each patch was selected. 
        All items were correctly classified by SUWR.
        }
        \label{fig:fashion-mnist}
        \vspace{-0.75\baselineskip}
    \end{center}
\end{figure*}

\begin{figure*}[t]
    \begin{center}
        \includegraphics[width=0.92\textwidth]{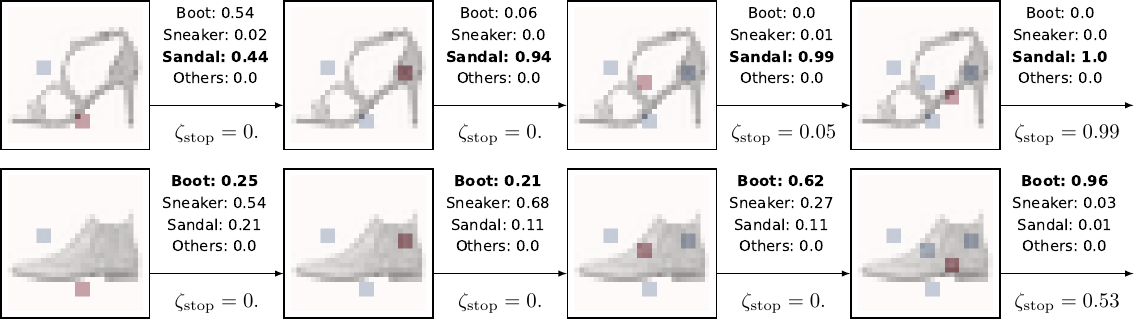}
        \vspace{-0.75\baselineskip}
        \caption{
        Narrative explanations derived from the SUWR inference process for a sandal (top) and boot (bottom) from fashion-MNIST.
        Step $t=2$ up to $t=5$ are visualized, red squares indicate patches selected in that step, blue squares those selected in previous steps.
        }
        \label{fig:shoes}
        \vspace{-1\baselineskip}
    \end{center}
\end{figure*}

\mpara{Conclusion}.
Our results on the synthetic benchmark reveal that \ac{SUWR} reaches higher predictive performance than the baselines.
We believe this is the case because leakage makes local feature selection methods more prone to overfitting, from which \ac{SUWR} is unaffected.
Furthermore, it also appears that \ac{SUWR} selects nearly all relevant features while excluding more irrelevant features than baseline methods.

\section{Experiment 3: MNIST Digits and Fashion}
\label{sec:experiments-MNIST}
\mpara{Setup}.
Finally, we evaluate \ac{SUWR} on an image classification task on two datasets: digits-MNIST~\citep{digit-mnist} and fashion-MNIST~\citep{fashion-mnist}.
Both datasets consist of 28$\times$28 (784) pixel images and each image is annotated by one of ten classes, indicating either which digit or which type of fashion item is in the image.
Because individual pixels are difficult to see in visualizations, we let the methods select 3$\times$3 patches of pixels on the fashion dataset.
As a result, the produced selection masks are much easier to interpret as selected pixels are less scattered.

\mpara{Methods}.
We omit L2X and INVASE from this comparison due to their extremely unrealistic and unfaithful behavior in a previous study by \citet{DBLP:conf/aistats/explanation-encode-pred} (e.g., $96\%$ accuracy while selecting a single pixel).
Despite its leakage, we do include REAL-X since its introduction was motivated with its performance on digits-MNIST~\citep{DBLP:conf/aistats/explanation-encode-pred}.
Additionally, we include the \ac{CAE}~\citep{cae}, a state-of-the-art \emph{global} feature selection method, and a predictor trained without any feature selection.
More details are given in Appendix~\ref{appendix:mnist-datasets}.

\begin{figure}
    \begin{center}
        \includegraphics[width=1\columnwidth, trim=15pt 20pt 10pt 0pt]{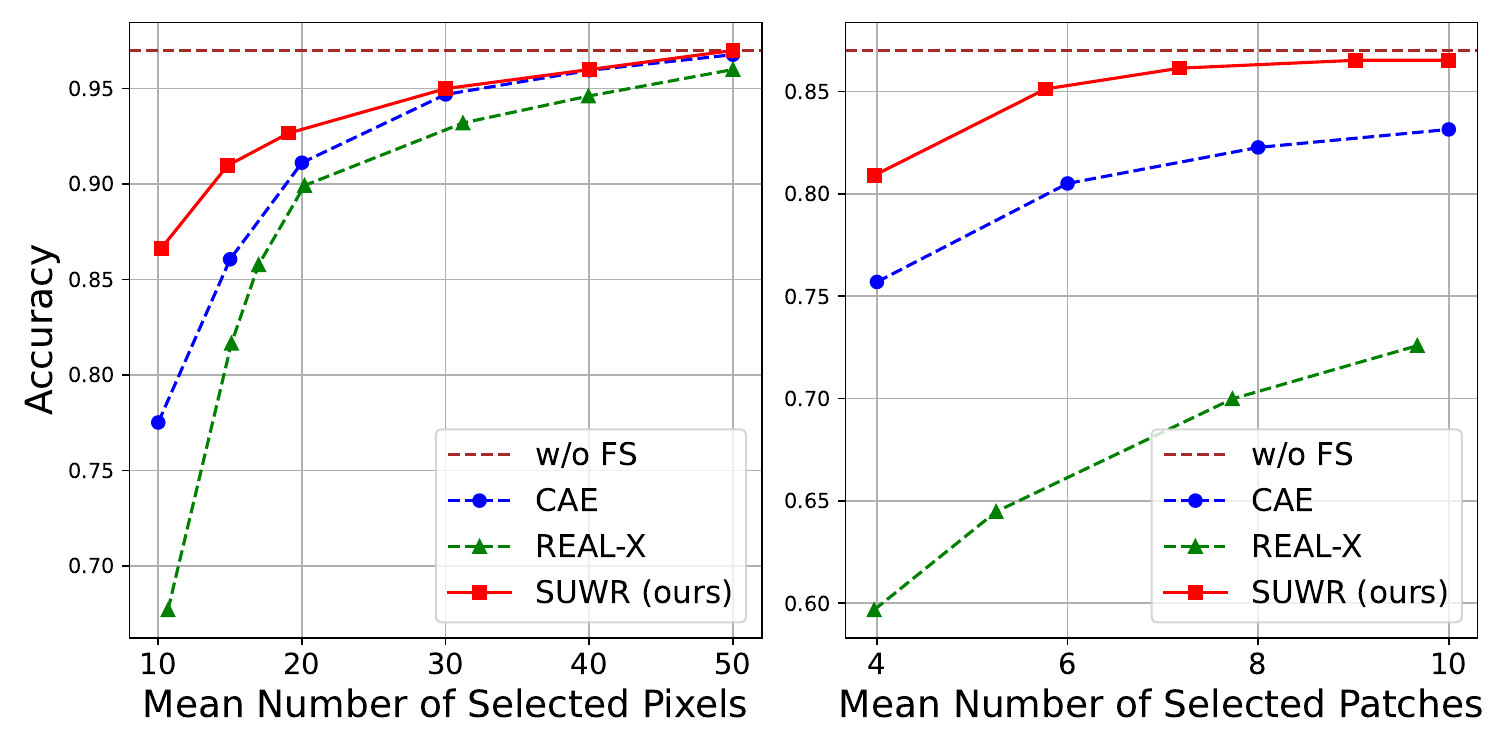}
        \vspace{-1.4\baselineskip}
        \caption{Results on MNIST: digits (left) and fashion (right).
        }
        \label{fig:mnist-results}
    \end{center}
    \vspace{-1.5\baselineskip}
\end{figure}

\mpara{Results}.
Figure~\ref{fig:mnist-results} displays the performance curves of the methods in terms of accuracy and the number of selected pixels or patches on the test sets.
We see that \ac{SUWR} consistently outperforms both CAE and REAL-X on both datasets, and even approximates the performance of the baseline without feature selection while only selecting a fraction of the features.
Admittedly, on digits-MNIST, the difference between \ac{SUWR} and CAE becomes marginal when more than thirty pixels are selected, indicating that local selection is less beneficial on this dataset.
In contrast, on fashion-MNIST, the differences between \ac{SUWR}, CAE and REAL-X are considerably large; for instance, CAE with ten patches does not yet achieve the performance that \ac{SUWR} reaches with just six patches.
Surprisingly, REAL-X consistently has considerably lower performance than both SUWR and CAE.
In other words, the local feature selection of REAL-X does substantially worse than even the global selections of CAE.
We speculate REAL-X is overfitting due to leakage, and additionally, that its intentional injection of noise during optimization hinders its performance.

\mpara{Conclusion}.
Our results on the MNIST datasets reveal that SUWR provides substantially better performance curves than REAL-X and CAE.
Thereby, SUWR shows that local feature selection \emph{without leakage} can provide considerably higher performance than global feature selection. %

\mpara{Interpretability}.
Lastly, we discuss several examples that illustrate how SUWR makes predictions more interpretable. %
Figure~\ref{fig:fashion-mnist} displays the selection masks for several items in fashion-MNIST, and the order in which patches were selected.
We see that the placement, order and number of selected patches highly varies per image.
Because SUWR has no leakage, we are certain that no features outside of the selected patches were used for prediction.
Thus, i.e., we know that the trousers were correctly classified based only on two patches. %
Similarly, the bag was classified using only four patches: three on its edges and an empty patch on the top.
While these insights may be surprising, they are provenly faithful and thus provide an accurate account of the complete information that \ac{SUWR} used for its predictions.

Figure~\ref{fig:fashion-mnist} illustrates several steps in the SUWR inference process for a sandal and a boot.
Besides what patches are selected per step, we also see how predictions and stop probabilities change as more features are selected.
This brings numerous interesting insights; e.g., the differences in predictions between the two items at $t=2$ can be attributed to a single pixel (top-left of the bottom patch).
Additionally, we see that the predictor is already correct about the sandal after the third patch, but \ac{SUWR} decides to select more features for more certainty.
To the best of our knowledge, SUWR is the first local feature selection method that provides narrative explanations that are  guaranteed to be faithful.

\vspace{-0.4\baselineskip}

\section{Conclusion}

This work has provided the first formal approach to feature and label leakage, which cause local feature selection methods to provide misleading explanations of what information predictions are based on.
We derived the necessary and sufficient conditions for leakage and introduced the first methods that are guaranteed to have no leakage: a linear programming method and \ac{SUWR}.
Our experimental results reveal that existing state-of-the-art methods are all subject to leakage, in addition to being misleading, this also appears to make them more prone to overfitting.
In contrast, our results indicate that \ac{SUWR} combines high selection sparsity with high predictive accuracy, outperforming all our baselines across several benchmarks.
Uniquely, the step-by-step SUWR process can be used as a narrative explanation itself.
The SUWR approach is generic and easily extendable, we believe it can serve as a strong foundation for future work on faithful interpretable \ac{ML} predictions with theoretical guarantees.

In particular, future work could consider methods to scale the SUWR approach to high-dimensional data.
For instance, by developing model architectures that can be applied efficiently in the SUWR framework.
Alternatively, one could investigate whether our definitions of leakage could be provide a basis for novel indicators of feature importance.
In order to promote future extentions of our work, we have made the implementations of our method and experiments publicly available at \url{https://github.com/GarfieldLyu/SUWR}.

\balance

\section*{Acknowledgements}

This work is partially supported by German Research Foundation (DFG), under the Project IREM with grant No.\,AN 996/1-1, and by the Netherlands Organisation for Scientific Research (NWO) under grant No.\ VI.Veni.222.269.

\section*{Impact Statement}

This paper makes a significant contribution to the field of interpretable machine learning, which is crucial for the development of transparent and hence responsible machine learning models. 
We show that our methods are versatile (covering at least two data types -- tabular data and images) and can be applied to numerous applications, from healthcare to finance.
Our research provides the theoretical foundation for further advancements in creating models that are not only effective but also intrinsically transparent and thus promote accountability. 
In an era where algorithmic decisions have profound impacts on individuals and societies, the methodologies presented in this paper ensure that these systems can be scrutinized and understood by stakeholders, thereby fostering trust and facilitating the broader adoption of AI technologies in sensitive and impactful domains.

\bibliography{reference}

\begin{thebibliography}{40}
\providecommand{\natexlab}[1]{#1}
\providecommand{\url}[1]{\texttt{#1}}
\expandafter\ifx\csname urlstyle\endcsname\relax
  \providecommand{\doi}[1]{doi: #1}\else
  \providecommand{\doi}{doi: \begingroup \urlstyle{rm}\Url}\fi

\bibitem[Arik \& Pfister(2021)Arik and Pfister]{DBLP:conf/aaai/tabnet}
Arik, S.~{\"{O}}. and Pfister, T.
\newblock Tabnet: Attentive interpretable tabular learning.
\newblock In \emph{Thirty-Fifth {AAAI} Conference on Artificial Intelligence,
  {AAAI} 2021, Thirty-Third Conference on Innovative Applications of Artificial
  Intelligence, {IAAI} 2021, The Eleventh Symposium on Educational Advances in
  Artificial Intelligence, {EAAI} 2021, Virtual Event, February 2-9, 2021},
  pp.\  6679--6687. {AAAI} Press, 2021.

\bibitem[Bal{\i}n et~al.(2019)Bal{\i}n, Abid, and Zou]{cae}
Bal{\i}n, M.~F., Abid, A., and Zou, J.
\newblock Concrete autoencoders: Differentiable feature selection and
  reconstruction.
\newblock In \emph{International conference on machine learning}, pp.\
  444--453. PMLR, 2019.

\bibitem[Bastings et~al.(2019)Bastings, Aziz, and
  Titov]{bastings2019rartionale}
Bastings, J., Aziz, W., and Titov, I.
\newblock Interpretable neural predictions with differentiable binary
  variables.
\newblock In \emph{Proceedings of the 57th Annual Meeting of the Association
  for Computational Linguistics}, pp.\  2963--2977, 2019.

\bibitem[Bishop \& Nasrabadi(2006)Bishop and Nasrabadi]{bishop2006pattern}
Bishop, C.~M. and Nasrabadi, N.~M.
\newblock \emph{Pattern recognition and machine learning}, volume~4.
\newblock Springer, 2006.

\bibitem[Chen et~al.(2022)Chen, He, Narasimhan, and
  Chen]{chen2022RationaleRobustness}
Chen, H., He, J., Narasimhan, K., and Chen, D.
\newblock Can rationalization improve robustness?
\newblock In \emph{North American Chapter of the Association for Computational
  Linguistics (NAACL)}, 2022.

\bibitem[Chen et~al.(2018)Chen, Song, Wainwright, and Jordan]{l2x}
Chen, J., Song, L., Wainwright, M., and Jordan, M.
\newblock Learning to explain: An information-theoretic perspective on model
  interpretation.
\newblock In \emph{International Conference on Machine Learning}, pp.\
  883--892. PMLR, 2018.

\bibitem[Covert et~al.(2023)Covert, Qiu, Lu, Kim, White, and
  Lee]{greedy-selection}
Covert, I.~C., Qiu, W., Lu, M., Kim, N.~Y., White, N.~J., and Lee, S.-I.
\newblock Learning to maximize mutual information for dynamic feature
  selection.
\newblock In \emph{International Conference on Machine Learning}, pp.\
  6424--6447. PMLR, 2023.

\bibitem[Dabkowski \& Gal(2017)Dabkowski and Gal]{dabkowski2017AEM}
Dabkowski, P. and Gal, Y.
\newblock Real time image saliency for black box classifiers.
\newblock \emph{Advances in neural information processing systems}, 30, 2017.

\bibitem[Dantzig(1963)]{dantzig1963linear}
Dantzig, G.
\newblock \emph{Linear programming and extensions}.
\newblock Princeton university press, 1963.

\bibitem[DeYoung et~al.(2020)DeYoung, Jain, Rajani, Lehman, Xiong, Socher, and
  Wallace]{deyoung2020eraser}
DeYoung, J., Jain, S., Rajani, N.~F., Lehman, E., Xiong, C., Socher, R., and
  Wallace, B.~C.
\newblock Eraser: A benchmark to evaluate rationalized nlp models.
\newblock In \emph{Proceedings of the 58th Annual Meeting of the Association
  for Computational Linguistics}, pp.\  4443--4458, 2020.

\bibitem[Du et~al.(2019)Du, Liu, and Hu]{du2019techniques}
Du, M., Liu, N., and Hu, X.
\newblock Techniques for interpretable machine learning.
\newblock \emph{Communications of the ACM}, 63\penalty0 (1):\penalty0 68--77,
  2019.

\bibitem[Gurrapu et~al.(2023)Gurrapu, Kulkarni, Huang, Lourentzou, Freeman, and
  Batarseh]{rationale-survey}
Gurrapu, S., Kulkarni, A., Huang, L., Lourentzou, I., Freeman, L.~J., and
  Batarseh, F.~A.
\newblock Rationalization for explainable {NLP:} {A} survey.
\newblock \emph{CoRR}, abs/2301.08912, 2023.
\newblock \doi{10.48550/ARXIV.2301.08912}.

\bibitem[Jacovi \& Goldberg(2021)Jacovi and
  Goldberg]{interpret-social-attribution}
Jacovi, A. and Goldberg, Y.
\newblock Aligning faithful interpretations with their social attribution.
\newblock \emph{Transactions of the Association for Computational Linguistics},
  9:\penalty0 294--310, 2021.

\bibitem[Jang et~al.(2017)Jang, Gu, and Poole]{gumbel-softmax}
Jang, E., Gu, S., and Poole, B.
\newblock Categorical reparameterization with gumbel-softmax.
\newblock In \emph{5th International Conference on Learning Representations,
  {ICLR} 2017, Toulon, France, April 24-26, 2017, Conference Track
  Proceedings}. OpenReview.net, 2017.

\bibitem[Jethani et~al.(2021)Jethani, Sudarshan, Aphinyanaphongs, and
  Ranganath]{DBLP:conf/aistats/explanation-encode-pred}
Jethani, N., Sudarshan, M., Aphinyanaphongs, Y., and Ranganath, R.
\newblock Have we learned to explain?: How interpretability methods can learn
  to encode predictions in their interpretations.
\newblock In Banerjee, A. and Fukumizu, K. (eds.), \emph{The 24th International
  Conference on Artificial Intelligence and Statistics, {AISTATS} 2021, April
  13-15, 2021, Virtual Event}, volume 130 of \emph{Proceedings of Machine
  Learning Research}, pp.\  1459--1467. {PMLR}, 2021.

\bibitem[Jethani et~al.(2023)Jethani, Saporta, and Ranganath]{fastshap-KL}
Jethani, N., Saporta, A., and Ranganath, R.
\newblock Don’t be fooled: label leakage in explanation methods and the
  importance of their quantitative evaluation.
\newblock In \emph{International Conference on Artificial Intelligence and
  Statistics}, pp.\  8925--8953. PMLR, 2023.

\bibitem[Kaufman et~al.(2012)Kaufman, Rosset, Perlich, and
  Stitelman]{kaufman2012leakage}
Kaufman, S., Rosset, S., Perlich, C., and Stitelman, O.
\newblock Leakage in data mining: Formulation, detection, and avoidance.
\newblock \emph{ACM Transactions on Knowledge Discovery from Data (TKDD)},
  6\penalty0 (4):\penalty0 1--21, 2012.

\bibitem[LeCun et~al.(2010)LeCun, Cortes, and Burges]{digit-mnist}
LeCun, Y., Cortes, C., and Burges, C.
\newblock Mnist handwritten digit database.
\newblock \emph{ATT Labs [Online]. Available:
  http://yann.lecun.com/exdb/mnist}, 2, 2010.

\bibitem[Lee et~al.(2021)Lee, Imrie, and van~der Schaar]{lee2021self}
Lee, C., Imrie, F., and van~der Schaar, M.
\newblock Self-supervision enhanced feature selection with correlated gates.
\newblock In \emph{International Conference on Learning Representations}, 2021.

\bibitem[Lei et~al.(2016)Lei, Barzilay, and Jaakkola]{lei2016rationalizing}
Lei, T., Barzilay, R., and Jaakkola, T.
\newblock Rationalizing neural predictions.
\newblock In \emph{Proceedings of the 2016 Conference on Empirical Methods in
  Natural Language Processing}, pp.\  107--117, 2016.

\bibitem[Lemhadri et~al.(2021)Lemhadri, Ruan, Abraham, and
  Tibshirani]{lemhadri2021lassonet}
Lemhadri, I., Ruan, F., Abraham, L., and Tibshirani, R.
\newblock Lassonet: A neural network with feature sparsity.
\newblock \emph{Journal of Machine Learning Research}, 22\penalty0
  (127):\penalty0 1--29, 2021.

\bibitem[Li et~al.(2017)Li, Cheng, Wang, Morstatter, Trevino, Tang, and
  Liu]{li2017feature}
Li, J., Cheng, K., Wang, S., Morstatter, F., Trevino, R.~P., Tang, J., and Liu,
  H.
\newblock Feature selection: A data perspective.
\newblock \emph{ACM computing surveys (CSUR)}, 50\penalty0 (6):\penalty0 1--45,
  2017.

\bibitem[Li \& Oliva(2021)Li and Oliva]{activate-selection}
Li, Y. and Oliva, J.
\newblock Active feature acquisition with generative surrogate models.
\newblock In Meila, M. and Zhang, T. (eds.), \emph{Proceedings of the 38th
  International Conference on Machine Learning, {ICML} 2021, 18-24 July 2021,
  Virtual Event}, volume 139 of \emph{Proceedings of Machine Learning
  Research}, pp.\  6450--6459. {PMLR}, 2021.

\bibitem[Lundberg \& Lee(2017)Lundberg and Lee]{shap}
Lundberg, S.~M. and Lee, S.-I.
\newblock A unified approach to interpreting model predictions.
\newblock \emph{Advances in neural information processing systems}, 30, 2017.

\bibitem[Martins \& Astudillo(2016)Martins and Astudillo]{sparsemax}
Martins, A. and Astudillo, R.
\newblock From softmax to sparsemax: A sparse model of attention and
  multi-label classification.
\newblock In \emph{International conference on machine learning}, pp.\
  1614--1623. PMLR, 2016.

\bibitem[Molnar(2020)]{molnar2020interpretable}
Molnar, C.
\newblock \emph{Interpretable machine learning}.
\newblock Lulu.com, 2020.

\bibitem[Paranjape et~al.(2020)Paranjape, Joshi, Thickstun, Hajishirzi, and
  Zettlemoyer]{DBLP:conf/emnlp/rationaleKLdiv}
Paranjape, B., Joshi, M., Thickstun, J., Hajishirzi, H., and Zettlemoyer, L.
\newblock An information bottleneck approach for controlling conciseness in
  rationale extraction.
\newblock In Webber, B., Cohn, T., He, Y., and Liu, Y. (eds.),
  \emph{Proceedings of the 2020 Conference on Empirical Methods in Natural
  Language Processing, {EMNLP} 2020, Online, November 16-20, 2020}, pp.\
  1938--1952. Association for Computational Linguistics, 2020.
\newblock \doi{10.18653/V1/2020.EMNLP-MAIN.153}.
\newblock URL \url{https://doi.org/10.18653/v1/2020.emnlp-main.153}.

\bibitem[Ribeiro et~al.(2016)Ribeiro, Singh, and Guestrin]{lime}
Ribeiro, M.~T., Singh, S., and Guestrin, C.
\newblock " why should i trust you?" explaining the predictions of any
  classifier.
\newblock In \emph{Proceedings of the 22nd ACM SIGKDD international conference
  on knowledge discovery and data mining}, pp.\  1135--1144, 2016.

\bibitem[Schwab \& Karlen(2019)Schwab and Karlen]{DBLP:conf/nips/CXPlain}
Schwab, P. and Karlen, W.
\newblock Cxplain: Causal explanations for model interpretation under
  uncertainty.
\newblock In Wallach, H.~M., Larochelle, H., Beygelzimer, A.,
  d'Alch{\'{e}}{-}Buc, F., Fox, E.~B., and Garnett, R. (eds.), \emph{Advances
  in Neural Information Processing Systems 32: Annual Conference on Neural
  Information Processing Systems 2019, NeurIPS 2019, December 8-14, 2019,
  Vancouver, BC, Canada}, pp.\  10220--10230, 2019.

\bibitem[Shrikumar et~al.(2017)Shrikumar, Greenside, and Kundaje]{deeplift}
Shrikumar, A., Greenside, P., and Kundaje, A.
\newblock Learning important features through propagating activation
  differences.
\newblock In \emph{International conference on machine learning}, pp.\
  3145--3153. PMLR, 2017.

\bibitem[Simonyan et~al.(2013)Simonyan, Vedaldi, and Zisserman]{saliency}
Simonyan, K., Vedaldi, A., and Zisserman, A.
\newblock Deep inside convolutional networks: Visualising image classification
  models and saliency maps.
\newblock \emph{arXiv preprint arXiv:1312.6034}, 2013.

\bibitem[Sutton et~al.(1999)Sutton, McAllester, Singh, and
  Mansour]{sutton1999policy}
Sutton, R.~S., McAllester, D., Singh, S., and Mansour, Y.
\newblock Policy gradient methods for reinforcement learning with function
  approximation.
\newblock \emph{Advances in neural information processing systems}, 12, 1999.

\bibitem[Tomsett et~al.(2020)Tomsett, Harborne, Chakraborty, Gurram, and
  Preece]{sanity-check}
Tomsett, R., Harborne, D., Chakraborty, S., Gurram, P., and Preece, A.
\newblock Sanity checks for saliency metrics.
\newblock In \emph{Proceedings of the AAAI conference on artificial
  intelligence}, volume~34, pp.\  6021--6029, 2020.

\bibitem[Vanderbei et~al.(2020)]{vanderbei2020linear}
Vanderbei, R.~J. et~al.
\newblock \emph{Linear programming}.
\newblock Springer, 2020.

\bibitem[Virtanen et~al.(2020)Virtanen, Gommers, Oliphant, Haberland, Reddy,
  Cournapeau, Burovski, Peterson, Weckesser, Bright, {van der Walt}, Brett,
  Wilson, Millman, Mayorov, Nelson, Jones, Kern, Larson, Carey, Polat, Feng,
  Moore, {VanderPlas}, Laxalde, Perktold, Cimrman, Henriksen, Quintero, Harris,
  Archibald, Ribeiro, Pedregosa, {van Mulbregt}, and {SciPy 1.0
  Contributors}]{2020SciPy-NMeth}
Virtanen, P., Gommers, R., Oliphant, T.~E., Haberland, M., Reddy, T.,
  Cournapeau, D., Burovski, E., Peterson, P., Weckesser, W., Bright, J., {van
  der Walt}, S.~J., Brett, M., Wilson, J., Millman, K.~J., Mayorov, N., Nelson,
  A. R.~J., Jones, E., Kern, R., Larson, E., Carey, C.~J., Polat, {\.I}., Feng,
  Y., Moore, E.~W., {VanderPlas}, J., Laxalde, D., Perktold, J., Cimrman, R.,
  Henriksen, I., Quintero, E.~A., Harris, C.~R., Archibald, A.~M., Ribeiro,
  A.~H., Pedregosa, F., {van Mulbregt}, P., and {SciPy 1.0 Contributors}.
\newblock {{SciPy} 1.0: Fundamental Algorithms for Scientific Computing in
  Python}.
\newblock \emph{Nature Methods}, 17:\penalty0 261--272, 2020.
\newblock \doi{10.1038/s41592-019-0686-2}.

\bibitem[Xiao et~al.(2017)Xiao, Rasul, and Vollgraf]{fashion-mnist}
Xiao, H., Rasul, K., and Vollgraf, R.
\newblock Fashion-mnist: a novel image dataset for benchmarking machine
  learning algorithms, 2017.

\bibitem[Yamada et~al.(2020)Yamada, Lindenbaum, Negahban, and
  Kluger]{yamada2020feature}
Yamada, Y., Lindenbaum, O., Negahban, S., and Kluger, Y.
\newblock Feature selection using stochastic gates.
\newblock In \emph{International conference on machine learning}, pp.\
  10648--10659. PMLR, 2020.

\bibitem[Yoon et~al.(2018)Yoon, Jordon, and van~der Schaar]{invase}
Yoon, J., Jordon, J., and van~der Schaar, M.
\newblock Invase: Instance-wise variable selection using neural networks.
\newblock In \emph{International Conference on Learning Representations}, 2018.

\bibitem[Zhang et~al.(2021)Zhang, Rudra, and Anand]{expred}
Zhang, Z., Rudra, K., and Anand, A.
\newblock Explain and predict, and then predict again.
\newblock In \emph{Proceedings of the 14th ACM International Conference on Web
  Search and Data Mining}, pp.\  418--426, 2021.

\bibitem[Zheng et~al.(2022)Zheng, Booth, Shah, and
  Zhou]{irrationality-rationales}
Zheng, Y., Booth, S., Shah, J., and Zhou, Y.
\newblock The irrationality of neural rationale models.
\newblock In \emph{Proceedings of the 2nd Workshop on Trustworthy Natural
  Language Processing (TrustNLP 2022)}, pp.\  64--73, Seattle, U.S.A., July
  2022. Association for Computational Linguistics.
\newblock \doi{10.18653/v1/2022.trustnlp-1.6}.

\end{thebibliography}
\bibliographystyle{icml2024}

\newpage
\appendix
\onecolumn

\section{Related Work}
\label{appendix:relatedwork}
The current mainstream lines of work in interpretable machine learning have been categorized into \textit{explaining trained models in post-hoc manner} ~\citep{lime,saliency,deeplift,shap} and \textit{building intrinsically explainable models} \citep{expred,l2x,invase}. 
Due to the discrepancies within post-hoc methods~\citep{sanity-check}, the latter has been increasingly advocated in recent years. 
In the neural era, one common way of building interpretable models is via input features.  
The main idea is to learn to select a small set of informative input features and use those features exclusively for the final prediction. Meanwhile, explanations come from the selected features, e.g., pixels for images and words for texts (we note that in language tasks this sort of method is more often referred as \textit{rationale} models~\citep{lei2016rationalizing,bastings2019rartionale,DBLP:conf/emnlp/rationaleKLdiv,chen2022RationaleRobustness}).  Thus, \textit{sparsity} (i.e., the number of selected features) and the final prediction performance are considered together to measure the model effectiveness and explainability~\citep{deyoung2020eraser}.

\mpara{Feature selection as explanation.} One challenge of feature selection is the scarcity of ground-truth labels to indicate the importance of features.
As a result, existing solutions learn to select features by jointly optimizing predictive performance and selection sparsity.
This type of joint training is referred as \textit{Joint amortized explanation methods}~\citep{dabkowski2017AEM,DBLP:conf/nips/CXPlain,DBLP:conf/aistats/explanation-encode-pred}.
The learnable selection and prediction function (selector and predictor) can be two separated models, e.g., as for CAE~\citep{cae}, L2X~\citep{l2x}, INVASE~\citep{invase} and REAL-X~\citep{DBLP:conf/aistats/explanation-encode-pred}, or components within a single model, e.g., as for TabNet~\citep{DBLP:conf/aaai/tabnet}.
For the former type, the training signals (e.g., the gradients or rewards) between the predictor and selector are propagated via Gumbel-relaxation~\citep{gumbel-softmax} or policy gradient.
For TabNet, the selection is generated by sparsemax activation~\citep{sparsemax} and thus trained by standard back-propogation. Additionally, CAE conducts global selection, and the others are local selection methods that vary selections per instance.

\mpara{Irrationality of local feature selection.}
Nevertheless, local selection methods, particularly the joint amortized methods have raised increasing concerns in recent works~\citep{interpret-social-attribution,irrationality-rationales,fastshap-KL}. 
They argue the selected features do not necessarily align with the true explanations, and thus \textit{unfaithful} to the model behaviors. 
Furthermore, ~\citet{DBLP:conf/aistats/explanation-encode-pred} showcased the selection mask can leak prediction to the predictor, and therefore achieve unrealistic high performance whether the selected features are relevant or not. 
As a remedy, they proposed REAL-X, which aims to prevent the predictor overfitting on the selector by injecting noise into the selection masks.
Our work shows that REAL-X is still subject to leakage (Section~\ref{sec:experiments-toy}), and to the best of our knowledge, we have proposed the first theoretically guaranteed solutions to this problem.

\mpara{Dynamic feature selection.}
Another tangentially relevant line of work is dynamic feature selection~\citep{activate-selection,greedy-selection}.
Similar to SUWR, some dynamic feature selection methods also conduct a greedy selection procedure without access to the full feature set.
However, dynamic feature selection is designed for settings where features are costly, and selection should be made to avoid the costs associated with retrieving specific feature values.
This is a very different purpose than our work, hence their methods are not designed to address \textit{leakage}, nor do they formally analyse interpretability for ML models.

\section{Necessary and sufficient conditions for feature selection without label or feature leakage}
\label{section:proofconditions}

Our formal proofs for the conditions for leakage will rely on two basic assumptions:
\begin{assumption}
\label{ass:sellabelind}
The choice of selector policy has no effect on the label distribution:
\begin{equation}
\forall (x, y, s^\text{in},s^\text{ex}) \in \Omega, \qquad
  p(y \mid x[s^\text{in}]) = p(y \mid x[s^\text{in}], \zeta).
\end{equation}
\end{assumption}
\begin{assumption}
\label{ass:selfeatind}
The choice of selector policy has no effect on the feature distribution:
\begin{equation}
\forall (x, y, s^\text{in},s^\text{ex}) \in \Omega, \qquad
  p(x[s^\text{ex}] \mid x[s^\text{in}]) = p(x[s^\text{ex}] \mid x[s^\text{in}], \zeta).
\end{equation}
\end{assumption}
Together, these assumptions entail that the \emph{natural} distribution of features and labels is not dependent on the feature selector, i.e., $\zeta$ does not have any effect on the feature and label frequencies in the data.
\begin{theorem}
\label{theorem:labelleakage}
A features selector does not have label leakage if and only if every probability for every possible feature selection does not depend on label values:
\begin{equation}
\begin{split}
& \Big( \neg \, \text{Label-Leakage}(\zeta) \Big)
    \\ &\hspace{1cm} \longleftrightarrow
    \Big(\forall (x, y, s^\text{in},s^\text{ex}) \in \Omega, \quad
     p( h[s^\text{in}] = 1, h[s^\text{ex}] = 0 \mid x[s^\text{in}], \zeta) = p( h[s^\text{in}] = 1, h[s^\text{ex}] = 0 \mid x[s^\text{in}], y, \zeta)
\Big).
\end{split}
\label{eq:labeltheorem}
\end{equation}
\end{theorem}
\begin{proof}
First, 
we take Eq.~\ref{eq:labelleakage} from Definition~\ref{def:labelleakage} and multiply both sides with $p(h[s^\text{in}] = 1, h[s^\text{ex}] = 0 \mid x[s^\text{in}], \zeta)$, to get the following:
\begin{equation}
\forall (x, y, s^\text{in},s^\text{ex}) \in \Omega,
\quad
p(y \mid x[s^\text{in}])p(h[s^\text{in}] = 1, h[s^\text{ex}] = 0 \mid x[s^\text{in}], \zeta) 
= p(y, h[s^\text{in}] = 1, h[s^\text{ex}] = 0 \mid x[s^\text{in}], \zeta).
\label{eq:labelleakageproofstep1}
\end{equation}
From Assumption~\ref{ass:sellabelind}, we have $p(y \mid x[s^\text{in}]) = p(y \mid x[s^\text{in}], \zeta)$, from Definition~\ref{def:omegaset} we know these values are positive, and thus we can divide each side of Eq.~\ref{eq:labelleakageproofstep1} by them:
\begin{equation}
\forall (x, y, s^\text{in},s^\text{ex}) \in \Omega,
\quad
\frac{p(y \mid x[s^\text{in}])p(h[s^\text{in}] = 1, h[s^\text{ex}] = 0 \mid x[s^\text{in}], \zeta)}{p(y \mid x[s^\text{in}])}
=
\frac{p(y, h[s^\text{in}] = 1, h[s^\text{ex}] = 0 \mid x[s^\text{in}], \zeta)}{p(y \mid x[s^\text{in}], \zeta)}.
\label{eq:labelleakageproofstep2}
\end{equation}
Reformulating each side of the above equation, results in:
\begin{equation}
\begin{split}
\forall (x, y, s^\text{in},s^\text{ex}) \in \Omega, \quad
p(h[s^\text{in}] = 1, h[s^\text{ex}] = 0 \,|\, x[s^\text{in}], \zeta)
=
p(h[s^\text{in}] = 1, h[s^\text{ex}] = 0 \,|\, x[s^\text{in}], y, \zeta).
\end{split}
\label{eq:labelleakageproofstep3}
\end{equation}
Thereby, we have proven that the condition for label leakage in Eq.~\ref{eq:labelleakage} of Definition~\ref{def:labelleakage} implies the condition in Eq.~\ref{eq:labelleakageproofstep3}.
Since our derivation is still valid when reversed, it also proves Eq.~\ref{eq:labelleakageproofstep3} implies Eq.~\ref{eq:labelleakage}.
Therefore, the conditions are logically equivalent, this completes our proof.
\end{proof}

\begin{theorem}
\label{theorem:featleakage}
A features selector does not have feature leakage if and only if every probability for every possible feature selection does not depend on non-selected feature values:
\begin{equation}
\begin{split}
&\Big( \neg \, \text{Feature-Leakage}(\zeta) \Big)
\\& \hspace{0.5cm}
    \longleftrightarrow \Big(\forall (x, y, s^\text{in},s^\text{ex}) \in \Omega, \;\;\;
    p( h[s^\text{in}] = 1, h[s^\text{ex}] = 0 \mid x[s^\text{in}], \zeta) 
= p( h[s^\text{in}] = 1, h[s^\text{ex}] = 0 \mid x[s^\text{in}], x[s^\text{ex}], \zeta)
\Big).
\end{split}
\label{eq:feattheorem}
\end{equation}
\end{theorem}
\begin{proof}
Analogous to the proof for Theorem~\ref{theorem:labelleakage}, we begin by taking
 Eq.~\ref{eq:featureleakage} from Definition~\ref{def:featureleakage} and multiply both sides with $p(h[s^\text{in}] = 1, h[s^\text{ex}] = 0 \mid x[s^\text{in}], \zeta)$, to get the following:
\begin{equation}
\forall (x, y, s^\text{in},s^\text{ex}) \in \Omega,
 \quad
p(x[s^\text{ex}] \mid x[s^\text{in}])p(h[s^\text{in}] = 1, h[s^\text{ex}] = 0, \mid x[s^\text{in}], \zeta) 
 = p(x[s^\text{ex}], h[s^\text{in}] = 1, h[s^\text{ex}] = 0, \mid x[s^\text{in}], \zeta).
 \label{eq:featleakageproofstep1}
\end{equation}
From Assumption~\ref{ass:selfeatind}, we have $p(x[s^\text{ex}] \mid x[s^\text{in}]) = p(x[s^\text{ex}] \mid x[s^\text{in}], \zeta)$, from Definition~\ref{def:omegaset} we know these values are positive, and thus we can divide each side of Eq.~\ref{eq:featleakageproofstep1} by them:
\begin{equation}
\forall (x, y, s^\text{in},s^\text{ex}) \in \Omega, \;\;
\frac{p(x[s^\text{ex}] \mid x[s^\text{in}])p(h[s^\text{in}] = 1, h[s^\text{ex}] = 0, \mid x[s^\text{in}], \zeta) }{p(x[s^\text{ex}] \mid x[s^\text{in}])}
=
\frac{p(x[s^\text{ex}], h[s^\text{in}] = 1, h[s^\text{ex}] = 0, \mid x[s^\text{in}], \zeta)}{p(x[s^\text{ex}] \mid x[s^\text{in}], \zeta)}.
\label{eq:featleakageproofstep2}
\end{equation}
Reformulating each side of the above equation, results in:
\begin{equation}
\forall (x, y, s^\text{in},s^\text{ex})\! \in\! \Omega, \quad
p(h[s^\text{in}] = 1, h[s^\text{ex}] = 0 \mid x[s^\text{in}], \zeta)
=
p(h[s^\text{in}] = 1, h[s^\text{ex}] = 0 \mid x[s^\text{in}], x[s^\text{ex}], \zeta).
\label{eq:featleakageproofstep3}
\end{equation}
Thereby, we have proven that the condition for feature leakage in Eq.~\ref{eq:featureleakage} of Definition~\ref{def:labelleakage} implies the condition in Eq.~\ref{eq:featleakageproofstep3}.
Since our derivation is still valid when reversed, it also proves Eq.~\ref{eq:featleakageproofstep3} implies Eq.~\ref{eq:featureleakage}.
Therefore, the conditions are logically equivalent, this completes our proof.
\end{proof}

\section{Local Feature Selection with SUWR has no Leakage}
\label{appendix:methodproof}

\begin{theorem}
All \ac{SUWR} feature-selection policies have no leakage.
In other words, if the inference of a policy $\zeta$ is computable with \ac{SUWR} then it has no leakage according to Definition~\ref{def:leakage}.
\end{theorem}
\begin{proof}
If $\zeta$ is computable by the inference algorithm of \ac{SUWR}, then it performs at most $T$ steps to make a selection.
From Algorithm~\ref{alg:inference}, we see that the creation of a selection ends when a Bernoulli trail with a probability determined by $\zeta_\text{stop}$ succeeds.
Therefore, the probability $\zeta(h \mid x)$ can be written as an expectation over $T$ steps; let $q(t = i, h \mid x, \zeta)$ indicate the probability that \ac{SUWR} reaches step $t = i$ and with the mask $h$, we can then formulate $\zeta(h \mid x)$ as:
\begin{equation}
\zeta(h \mid x) = q(t = T, h \mid x, \zeta)  + \sum_{i = 0}^{T-1} q(t = i, h \mid x, \zeta) \zeta_\text{stop}^{t=i}(x \odot h).
\label{eq:appendix:suwrreform}
\end{equation}
In less formal terms, it is a sum over the probability of reaching each possible step and the mask being $h$ at that step multiplied with the probability of stopping at that step.
Thus, $q(t = i, h \mid x, \zeta)$ is the probability of \ac{SUWR} \emph{reaching} a step, \emph{not necessarily stopping} at that step.
Accordingly, in the first step ($t=0$), the mask is always the empty mask, therefore:
\begin{equation}
q(t = 0,  h = \mathbf{0} \mid x, \zeta) = 1, \qquad q(t = 0,  h \not= \mathbf{0} \mid x, \zeta) = 0.
\end{equation}
To keep our notation brief, we call a mask a subset of another mask if it selects the same or a subset of features:
\begin{equation}
h' \subseteq h \longleftrightarrow (\forall i, \; h'[i] = 1 \longrightarrow h[i] = 1).
\end{equation}
This enables us to give a short definition general definition of $q(t,  h  \mid x, \zeta)$ by using its recursive nature:
\begin{equation}
q(t,  h  \mid x, \zeta)
=
\begin{cases}
1, & \text{if } t = 0 \land h = \mathbf{0},\\
0, & \text{if } t = 0 \land h \not= \mathbf{0},\\
\sum\limits_{h' : h' \subseteq h}q(t-1,  h' \mid x, \zeta)
(1 - \zeta_\text{stop}^{t-1}(x \odot  h'))\sum\limits_{u \in \{0,1\}^{d} : h' + u = h \hspace{-1.75cm}}\zeta_\text{select}^{t-1}(u \mid x \odot  h')
, & \text{otherwise.}
\end{cases}
\label{eq:appendix:gfunction}
\end{equation}
Thus, when $t > 0$, the value of $q(t,  h  \mid x, \zeta)$ is a sum over the probability that the previous step ($t-1$) was reached with a subset of $h' \subseteq h$, and that the \ac{SUWR} process did not stop, and that a new feature mask $u$ was sampled such that $h = h' + u$.
This recursion ends when $t=0$ is reached.

Clearly, we can see from Eq.~\ref{eq:appendix:gfunction} that for $t=0$ the $q$ function is not conditioned on $x$:
\begin{equation}
q(t = 0,  h = \mathbf{0} \mid x, \zeta) = q(t = 0,  h = \mathbf{0} \mid \zeta), \qquad q(t = 0,  h \not= \mathbf{0} \mid x, \zeta) = q(t = 0,  h \not= \mathbf{0} \mid \zeta),
\end{equation}
and therefore:
\begin{equation}
q(t = 0,  h \mid x, \zeta) = q(t = 0,  h \mid \zeta).
\end{equation}
Similarly, at $t=1$ the following holds:
\begin{equation}
q(t = 1,  h \mid x, \zeta) = q(t = 0,  h = \mathbf{0} \mid \zeta) (1 - \zeta_\text{stop}^{t=0}(\emptyset))\zeta_\text{select}^{t=0}(h \mid \emptyset),
\end{equation}
and therefore:
\begin{equation}
q(t = 1,  h \mid x, \zeta) = q(t = 1,  h \mid \zeta).
\end{equation}
We can continue this pattern by considering Eq.~\ref{eq:appendix:gfunction}, where we can see that when $t>0$ the $\zeta_\text{stop}$ and $\zeta_\text{select}$ only take subsets of $h$ as input.
Similarly, through the recursion of $q$ only subsets of $h$ are given as input to $q$, therefore, the recursion cannot add a dependency on any feature value not selected by $h$.
Consequently, the value of $q(t,  h \mid x, \zeta)$ does not depend on any values of $x$ not selected by $h$:
\begin{equation}
q(t,  h[s^\text{in}] = 1, h[s^\text{ex}] = 0  \mid x, \zeta) = q(t,  h[s^\text{in}] = 1, h[s^\text{ex}] = 0  \mid x[s^\text{in}], \zeta).
\end{equation}
Finally, combining this result with Eq.~\ref{eq:appendix:suwrreform}, we see that the final stop probability also does not add a dependency on feature values not selected by $h$, therefore:
\begin{equation}
\zeta(h[s^\text{in}] = 1, h[s^\text{ex}] = 0 \mid x) = \zeta(h[s^\text{in}] = 1, h[s^\text{ex}] = 0 \mid x[s^\text{in}]).
\end{equation}
According to Corollary~\ref{corollary:mainresult}, this proves that $\zeta$ does not have leakage.
\end{proof}

\section{Conjecture: SUWR Describes any Selection Policy without Leakage under Full-Support Feature Distributions}
\label{appendix:onlymethodproof}

\begin{assumption}
\label{assumption:appendix:featureproduct}
The feature value distribution provides support for the Cartesian product of possible feature values.
In other words, if there is a positive probability that feature $x[i]$ has value $a$ and a positive probability that feature $x[j]$ has value $b$, then there is a positive probability that feature $x[i]$ has value $a$ \emph{and} feature $x[j]$ has value $b$ simultaneously:
\begin{equation}
\forall i,j,a,b, \quad \big( p(x[i] = a) > 0 \land p(x[j] = b) > 0 \big)
\longrightarrow p(x[i] = a, x[j] = b) > 0.
\label{eq:appendix:featureproduct}
\end{equation}
\end{assumption}

\acrodef{RDHD}{reversed directed Hasse diagram}

\begin{definition}
We define a \acfi{RDHD} \ac{SUWR} policy as a \ac{SUWR} policy where the maximum step is the number of features: $T=d$, and  $\zeta^{t}_\text{select}(x \odot h)$ is a distribution over all single features that have not been selected yet:
\begin{equation}
\zeta^t_\text{select}( u \mid x \odot h) 
\begin{cases}
\geq 0 & \text{if }\,  \big( \exists! i, \; u[i] = 1 \big) \land \big( \forall i \in \{1,2,\ldots,d\}, \, u[i] = 1 \rightarrow h[i] = 0 \big), \\
= 0 & \text{otherwise.}
\end{cases}
\end{equation}
Thereby, at each step $t$, the process either stops or a single feature is added to $h$.
As a result, the number of features selected by $h^t$ is always equal to $t$: $\sum_{i=1}^d h^t[i] = d$.
An example visualization of the possible steps of a \ac{RDHD} \ac{SUWR} policy for three features is shown in Figure~\ref{fig:appendix:diagram}.

The naming of this type of policy is inspired by the fact that the inference process of a \ac{RDHD} \ac{SUWR} policy can be visualized as  traversing over a Hasse diagram (e.g., in Figure~\ref{fig:appendix:diagram}).
Traditionally, Hasse diagrams are constructed from the complete set and are not directed.
In contrast, \ac{RDHD} \ac{SUWR} policies start with the empty set and explicitly only traverse in the direction where elements are added.
Hence, we name it after a \emph{reversed} and \emph{directed} version of the Hasse diagram.
\end{definition}

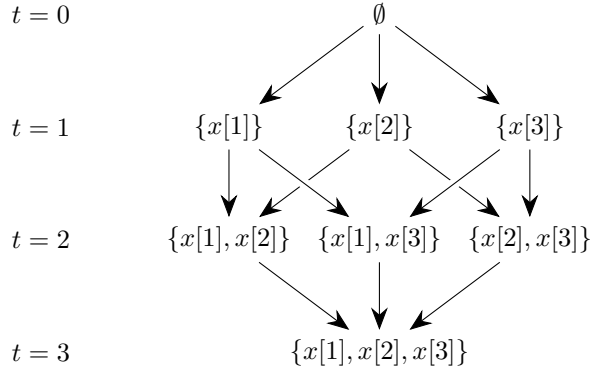
\begin{figure}[t]
\begin{center}
\begin{tikzpicture}
  \node (max) at (0,3) {$\emptyset$};
  \node (a) at (-2,1.5) {$\{x[1]\}$};
  \node (b) at (0,1.5) {$\{x[2]\}$};
  \node (c) at (2,1.5) {$\{x[3]\}$};
  \node (d) at (-2,0) {$\{x[1],x[2]\}$};
  \node (e) at (0,0) {$\{x[1],x[3]\}$};
  \node (f) at (2,0) {$\{x[2], x[3]\}$};
  \node (min) at (0,-1.5) {$\{x[1],x[2],x[3]\}$};
  \node (g) at (-4.5,3) {$t=0$};
  \node (h) at (-4.5,1.5) {$t=1$};
  \node (i) at (-4.5,0) {$t=2$};
  \node (j) at (-4.5,-1.5) {$t=3$};
  \draw [-{Stealth[length=3mm]}] (max) -- (a);
  \draw [-{Stealth[length=3mm]}] (max) -- (b);
  \draw [-{Stealth[length=3mm]}] (max) -- (c);
  \draw [-{Stealth[length=3mm]}] (a) -- (d);
   \draw [-{Stealth[length=3mm]}] (b) -- (d);
   \draw [-{Stealth[length=3mm]}] (b) -- (f);
   \draw [-{Stealth[length=3mm]}] (c) -- (f);
  \draw[preaction={draw=white, -,line width=4pt},-{Stealth[length=3mm]}] (a) -- (e);
  \draw[preaction={draw=white, -,line width=4pt},-{Stealth[length=3mm]}] (c) -- (e);
   \draw [-{Stealth[length=3mm]}] (d) -- (min);
   \draw [-{Stealth[length=3mm]}] (e) -- (min);
   \draw [-{Stealth[length=3mm]}] (f) -- (min);
\end{tikzpicture}
\end{center}
\caption{Visualization of all possible steps and transitions for a \ac{RDHD} SUWR policy when selecting from a set of three features.}
\label{fig:appendix:diagram}
\end{figure}

\begin{conjecture}
\label{theorem:appendix:onlymethod}
Under the assumption that the feature distribution supports the Cartesian product of possible feature values (Assumption~\ref{assumption:appendix:featureproduct}),
every feature selection policy $\zeta$ that has no leakage (Definition~\ref{def:leakage}) has an equivalent \ac{RDHD} \ac{SUWR} policy.
In other words, the set of all possible feature selection policies without leakage is a subset of the set of all possible \ac{RDHD} \ac{SUWR} policies.
\end{conjecture}
\noindent
\emph{Support.}
We will provide reasons to support that, under Assumption~\ref{assumption:appendix:featureproduct}, for any $\zeta$ without leakage, there exists a $\zeta^t_\text{stop}$ and $\zeta^t_\text{select}$ for a \ac{RDHD} \ac{SUWR} policy, such that $\zeta$ and the \ac{RDHD} \ac{SUWR} policy have an identical distribution over feature selections.

For this section, the same $q$ function is used as for Theorem~\ref{appendix:methodproof}, but to keep our notation short, we will use $q(h \mid x \odot  h)$ instead of $q(t,  h[s^\text{in}] = 1, h[s^\text{ex}] = 0  \mid x[s^\text{in}], \zeta_\text{stop}, \zeta_\text{select})$.
We can do this without loss of specificity since each $h$ can only occur at a single specific step: $t = \sum_{i=1}^d h[i]$, we only consider $q$ in the context of $\zeta_\text{stop}$ \& $\zeta_\text{select}$, and we have already proven that $q$ only depends on the features selected by $h$, i.e., $x \odot h$ (see Theorem~\ref{appendix:methodproof}).
In other words, we use $q(h \mid x \odot  h)$ as the probability that the \ac{SUWR} process at some point \emph{considers} mask $h$, this is not the probability that $h$ is selected.

This difference reveals the requirement on the \ac{SUWR} policy, the probability of considering $h$ should be equal to or greater than the probability to select $h$:
\begin{equation}
\forall h, x, \qquad p(x) > 0 \longrightarrow q(h \mid x \odot  h) > \zeta(h \mid x \odot h).
\label{eq:appendix:qminimum}
\end{equation}
This requirement exists because the probability of selecting $h$ in a \ac{RDHD} \ac{SUWR} policy is equal to:
\begin{equation}
\forall h, x, \qquad p(x) > 0 \longrightarrow \zeta(h \mid x \odot h) = q(h \mid x \odot  h) \zeta_\text{stop}(h \mid x \odot h).
\end{equation}
Therefore, the probabilities $\zeta_\text{stop}(h \mid x \odot h)$ have to be:
\begin{equation}
\forall h, x, \qquad p(x) > 0 \longrightarrow
\zeta_\text{stop}(h \mid x \odot h) = \frac{
\zeta(h \mid x \odot h)
}{
q(h  \mid x \odot h, \zeta_\text{stop}, \zeta_\text{select}),
}
\end{equation}
this is a valid probability, i.e., $\zeta_\text{stop}(h \mid x \odot h) \in [0,1]$, if Eq.~\ref{eq:appendix:qminimum} is true, i.e., $q(h \mid x \odot  h) > \zeta(h \mid x \odot h)$.

Thus, we have to choose $\zeta_\text{select}$ such that Eq.~\ref{eq:appendix:qminimum} is guaranteed to hold.
To keep our notation short, we denote $\zeta^t_\text{select}( i \mid x \odot h)$ for the selection of feature $i$, i.e., the sampling of a vector $u$ such that only element $i$ is one: $u[i] = 1$ and all other values are zero: $i \not= j \leftrightarrow u[j] = 0$, conditioned on the feature values of $x$ selected by $h$.

Our key insight is that every time a \ac{RDHD} \ac{SUWR} policy samples a feature, it is excluding a set of possible selections, which can no longer be reached afterwards.
Instead of thinking about how the \ac{SUWR} process includes features into its selection, we consider the possible feature selections it excludes through the addition of each feature.
The following set covers all masks that can no longer be reached after $i$ is sampled by \ac{SUWR} to be added to mask $h$:
\begin{equation}
\operatorname{excluded}(h , i) = \{
h' : h[i]=0 \land \forall j \in \{1,2,\ldots, d\}, h[j] =1 \rightarrow h'[j] = 1
\}.
\end{equation}
As we can see, each mask in $\operatorname{excluded}(h , i)$ makes the same selections as $h$, in addition to every other possible selection, that does not select $i$ as well.
We note that when the set is empty when $i$ has already been selected in $h$: $h[i] = 1 \longrightarrow \operatorname{excluded}(h , i) = \emptyset$.
Therefore, the probability that feature $i$ is added to selection $h$ should not exceed the following:
\begin{equation}
\underbrace{
q(h  \mid x \odot h, \zeta_\text{stop}, \zeta_\text{select}) (1 - \zeta_\text{stop}( i \mid x \odot h))  \zeta_\text{select}( i \mid x \odot h)
}_\text{probability of reaching $h$ and adding $i$}
\leq 1 -
\underbrace{
\max_{\{x[j] : h[j] = 0 \land i \not= j\}} \sum_{h' \in \operatorname{excluded}(h, i)} \zeta( h' \mid x \odot h)
}_\text{max. prob. of selections no longer accessible afterwards}.
\end{equation}
This leads to the following restriction on $\zeta_\text{select}$:
\begin{equation}
\zeta_\text{select}( i \mid x \odot h)
\leq
\frac{
1 -
\max_{\{x[j] : h[j] = 0 \land i \not= j\}} \sum_{h' \in \operatorname{excluded}(h, i)} \zeta( h' \mid x \odot h)
}{
q(h  \mid x \odot h, \zeta_\text{stop}, \zeta_\text{select})  (1 - \zeta_\text{stop}( i \mid x \odot h))
}
\label{eq:appendix:selectmax}
\end{equation}
This maximum restriction ensures that these selections remain reachable by the \ac{RDHD} \ac{SUWR} policy with the required probability.
Thereby ensuring the requirement in Eq.~\ref{eq:appendix:qminimum} is true.
Importantly, this maximum can be inferred without knowledge of feature values that are not selected in $h$, thus it can be incorporated without introducing leakage.

However, not selecting feature $i$ also excludes a possible selection.
Namely, the selection that is made by only adding feature $i$ to $h$, as this can no longer be reached after the addition of a different feature.
We denote this mask as $h^{+i}$:
\begin{equation}
h^{+i} \in \{0,1\}^d \quad \text{s.t.} \quad  h[i]=1 \land \forall j \in \{1,2,\ldots, d\}, i \not= j \leftrightarrow h'[j] = h[j].
\end{equation}
Therefore, the probability of reaching $h$ and selecting $i$ must be at least as great as the maximal possible probability of $h^{+i}$ conditioned on the feature values selected so far:
\begin{equation}
\max_{x[i]} \zeta( h^{+i} \mid x \odot h^{+i})
\leq q(h  \mid x \odot h, \zeta_\text{stop}, \zeta_\text{select})(1 - \zeta_\text{stop}( i \mid x \odot h)) \zeta_\text{select}( i \mid x \odot h).
\end{equation}
This results in the following restriction on $\zeta_\text{select}$:
\begin{equation}
\frac{
\max_{x[i]} \zeta( h^{+i} \mid x \odot h^{+i})
}{
q(h  \mid x \odot h, \zeta_\text{stop}, \zeta_\text{select})(1 - \zeta_\text{stop}( i \mid x \odot h))
}
\leq
\zeta_\text{select}( i \mid x \odot h).
\label{eq:appendix:selectmin}
\end{equation}
Again, we note that this restriction can be enforced without introducing leakage as the minimum value only depends on feature values that are not selected in $h$.

By combining the requirement in Eq.~\ref{eq:appendix:selectmin} and Eq.~\ref{eq:appendix:selectmax}, we see that we need the following requirement to be true:
\begin{equation}
\max_{x[i]} \zeta( h^{+i} \mid x \odot h^{+i})
+
\max_{\{x[j] : h[j] = 0 \land i \not= j\}} \sum_{h' \in \operatorname{excluded}(h, i)} \zeta( h' \mid x \odot h)
\leq
1.
\label{eq:appendix:probplusreq}
\end{equation}
Since if Eq.~\ref{eq:appendix:probplusreq} is not true, there is no value of $\zeta_\text{select}( i \mid x \odot h)$ that can satisfy both Eq.~\ref{eq:appendix:selectmin} and Eq.~\ref{eq:appendix:selectmax}.

We will now show that under Assumption~\ref{assumption:appendix:featureproduct}, the requirement in  Eq.~\ref{eq:appendix:probplusreq} is always guaranteed.%
\footnote{For comparison, Table~\ref{tab:appendix:suwrimpossible} displays an example where Assumption~\ref{assumption:appendix:featureproduct} is not true, and accordingly, Conjecture~\ref{theorem:appendix:onlymethod} does not hold.}
To start, we use the following to denote the feature values that maximize each of the maximum operations:
\begin{equation}
\begin{split}
x[i]^* &= \arg\max_{x[i]} \zeta( h^{+i} \mid x \odot h^{+i}),
\\
\{x[j]^*\} &= \max_{\{x[j] : h[j] = 0 \land i \not= j\}} \sum_{h' \in \operatorname{excluded}(h, i)} \zeta( h' \mid x \odot h).
\end{split}
\end{equation}
Importantly, feature $i$ is not selected by any mask in the excluded set: $\forall h' \in \operatorname{excluded}(h, i), \; h[i] = 0$.
Therefore, there is no overlap between $x[i]^*$ and $\{x[j]^*\}$, this means that a possible vector of feature values exists that includes $x[i]^*$, $\{x[j]^*\}$ and $x \odot h$.
We denote this combination of possible values as:
\begin{equation}
\exists x^* : x^*[i] = x[i]^* \land \big( \forall x[j]^*,\, x^*[j] = x[j]^* \big)
\land x \odot h = x^* \odot h
.
\end{equation}
By the definition of $x^*$, $x[i]^*$ and $\{x[j]^*\}$, this vector maximizes both parts of the left side of Eq.~\ref{eq:appendix:probplusreq}:
\begin{equation}
\begin{split}
\zeta( h^{+i} \mid x^* \odot h^{+i})
&= \max_{x[i]} \zeta( h^{+i} \mid x \odot h^{+i}),
\\
\sum_{h' \in \operatorname{excluded}(h, i)} \zeta( h' \mid x^* \odot h)
&=
\max_{\{x[j] : h[j] = 0 \land i \not= j\}} \sum_{h' \in \operatorname{excluded}(h, i)} \zeta( h' \mid x \odot h)
.
\end{split}
\end{equation}
Assumption~\ref{assumption:appendix:featureproduct} states that every possible combination of feature values is supported by the feature distribution, therefore: $p(x^*) > 0$.
$\zeta$ is a valid probability distribution over all possible feature masks.
For every possible value of $x$, this means the sum of all probabilities of all masks cannot be greater than one.
Therefore, the same goes for this subset of masks:
\begin{equation}
    p(x^*) > 0 \longrightarrow \zeta( h^{+i} \mid x^* \odot h^{+i}) + \sum_{h' \in \operatorname{excluded}(h, i)} \zeta( h' \mid x^* \odot h) \leq 1
    .
\end{equation}
Consequently, the requirement in Eq.~\ref{eq:appendix:probplusreq} is guaranteed to hold under Assumption~\ref{assumption:appendix:featureproduct}, and therefore, there always exists a value for $\zeta_\text{select}( i \mid x \odot h)$ that can satisfy both Eq.~\ref{eq:appendix:selectmin} and Eq.~\ref{eq:appendix:selectmax}.

Unfortunately, this does not provide a complete proof, since there is an additional requirement that we were unable to prove.
Namely, Eq.~\ref{eq:appendix:selectmin} can only hold if the following is true:
\begin{equation}
\frac{
\max_{x[i]} \zeta( h^{+i} \mid x \odot h^{+i})
}{
q(h  \mid x \odot h, \zeta_\text{stop}, \zeta_\text{select})(1 - \zeta_\text{stop}( i \mid x \odot h))
}
\leq
1,
\end{equation}
since otherwise, Eq.~\ref{eq:appendix:selectmin} implies that $\zeta_\text{select}( i \mid x \odot h) > 1$ which would make it an invalid policy.
A simple reformulation reveals that this is a requirement on $q$:
\begin{equation}
q(h  \mid x \odot h, \zeta_\text{stop}, \zeta_\text{select})
\leq
\frac{
\max_{x[i]} \zeta( h^{+i} \mid x \odot h^{+i})
}{
1 - \zeta_\text{stop}( i \mid x \odot h)
}.
\label{eq:finalreq}
\end{equation}
In other words, the probability of $h$ being considered conditioned on $x \odot h$, $\zeta_\text{stop}$ and $\zeta_\text{select}$ needs to be great enough to provide enough probability mass for both the maximum possible probability of $h$ and $h^{+i}$.
For very small problems with two binary features, we are able to find a closed-form solution that gaurantees this.
Unfortunately, we were unable to extend this approach to a more generic setting.
Nonetheless, it appears that satisfying both Eq.~\ref{eq:appendix:selectmin} and Eq.~\ref{eq:appendix:selectmax} also guarantees Eq.~\ref{eq:finalreq}, but until this is proven our claim can only remain a conjecture.

\begin{table}[t]
\centering %
\caption{Example of a feature selection policy \emph{without} leakage that is \emph{impossible} to compute with \ac{SUWR}.
This happens because the feature distribution does not support the Cartesian product of possible feature values, as stated in Assumption~\ref{assumption:appendix:featureproduct}.
In this example, knowing that $x[1] = 1$ means one also knows $x[2]=0$ and $x[3] = 0$, therefore, these selections can be safely removed once $x[1] = 1$ is known, without introducing leakage.
Since \ac{SUWR} is agnostic to the underlying feature distribution, it does not use this property to enable the removal of features after their selection.
Consequently, the displayed policy cannot be executed through the \ac{SUWR} algorithm.
(See Table~\ref{tab:featleak} for an explanation of the notation).
}
\label{tab:appendix:suwrimpossible}

 \setlength{\tabcolsep}{5pt}
\begin{tabular}{c | c c c | c c c | c c c | c }
\toprule
$p(x, y, h )$ &
$x[1]$ & $x[2]$ & $x[3]$ &
$h[1]$ &  $h[2]$ & $h[2]$ &
$(x\odot h)$[1] & 
$(x \odot h)$[2] &  
$(x \odot h)$[3] &
y \\
\midrule
   $0.333\ldots$ & 1 & 0 & 0 & 1 & 0 & 0 & 1 &  &   & 2 \\
   $0.333\ldots$ & 0 & 1 & 0 & 0 & 1 & 0 &  & 1 &   & 1 \\
   $0.333\ldots$ & 0 & 0 & 1 & 0 & 0 & 1 &  &  & 1  & 0 \\
   \bottomrule
\end{tabular}

\end{table}

\section{Details on the Linear Programming Approach}
\label{appendix:linearprogram}

For our linear programming approach, we assume that the problem is fully known, thus complete knowledge of $p(x,y)$ is available.
In addition, we assume that the set of possible feature and label values is finite and iterable.
As a result, the optimal predictor $f^*$ can be treated as a lookup table that stores the optimal prediction per possible selected feature values.
For simplicity, we assume that the optimal prediction value is the expected label conditioned on the selected feature values:
\begin{equation}
f^*( x \odot h) = \mathbb{E}_x[y \mid x \odot h]
= \sum_{x' : x'\odot h = x \odot h} p(x') \sum_{y} p(y \mid x') y
= \sum_{x' : x'\odot h = x \odot h} p(x') \sum_{y} p(y \mid x') y.
\end{equation}
Therefore, we only have to find the optimal selector policy $\zeta$.
Our linear programming approach poses the search as a constrained minimization problem in the following form~\cite{dantzig1963linear, vanderbei2020linear}:
\begin{equation}
\min_\theta c^T\theta \quad \text{s.t.} \quad A\theta = b \; \land \; \mathbf{0} \leq \theta \leq \mathbf{1},
\label{eq:appendix:linprogloss}
\end{equation}
where $\theta$ is a vector where each element represents the conditional probability of a selection $\zeta(h \mid x)$.
The remainder of this section will show how the vectors $b$ and $c$ and matrix $A$ can be constructed so that this minimization problem is equivalent to selection policy optimization.

To start, we will show how $c$ and $\theta$ can be chosen so that $c^T\theta = \mathcal{L}(\zeta, f^*)$ (cf.\ Eq.~\ref{eq:genericloss}).
Importantly, we want our selection policy to have no leakage, as discussed in Section~\ref{sec:conditions} this means that:
\begin{equation}
\forall (x, x', h), \quad x \odot h = x' \odot h \longrightarrow \zeta(h \mid x) = \zeta(h \mid x').
\end{equation}
Therefore, we only have to find a single conditional probability $\zeta(h \mid x)$ for every unique $x \odot h$ value.
Thus, the size of vector $x$ is going to be the number of unique possible selected feature values, where each element corresponds to a single $x \odot h$ and contains the value for all corresponding $\zeta(h \mid x)$ values.
To see that our loss can be rewritten as a dot product with such a vector, we rewrite it as follows:
\begin{equation}
\begin{split}
    \mathcal{L}(\zeta, f^*)
    &=
    \sum_{x,y} p(x,y) \sum_{h} \big[ \zeta(x \odot h) L(f^*( x \odot h), y) + \lambda \lVert h \rVert \big]
    \\ &=
    \sum_{x \odot h} \zeta(h \mid x) \sum_{x' : x'\odot h = x \odot h} p(x') \Big( \lambda \lVert h \rVert + \sum_y p(y) L(f^*( x \odot h), y) \Big),
\end{split}
\end{equation}
where the summation $\sum_{x \odot h}$ sums over every possible value of $x \odot h$ once.
In other words, if multiple feature values result in the same selected feature values e.g., $x \odot h = x' \odot h$, only one of them is considered in the $\sum_{x \odot h}$ sum.
From the above reformulation, we see that for $c^T\theta = \mathcal{L}(\zeta, f^*)$ we require:
\begin{equation}
\forall (x, h), \, \exists! i \in \mathbb{N}_{>0}, \quad
\theta_i = \zeta(h \mid x) \land
c_i = \sum_{x' : x'\odot h = x \odot h} p(x') \Big( \lambda \lVert h \rVert + \sum_y p(y) L(f^*( x \odot h), y) \Big).
\label{eq:appendix:linprogcrequirement}
\end{equation}
Algorithm~\ref{alg:linprog} shows how we construct $c$ accordingly: first a mapping is made for every possible selected feature value ($x \odot h$) to an index on $c$, next the value of each element of $c$ is computed following  Eq.~\ref{eq:appendix:linprogcrequirement} and stored in the corresponding position.

Besides minimizing $\mathcal{L}$, it is important that the $\zeta$ is a valid probability distribution.
To be more precise, for all possible values of the full set of features $x$, $\zeta$ should produce a valid distribution over all possible selections ($\zeta(h \mid x)$).
We can express this formally in the following manner:
\begin{equation}
\forall x, \; p(x) > 0
\longrightarrow \Big(
\sum_{h \in \zeta(x)}
\zeta( h \mid x ) 
=
\sum_{s^\text{in}, s^\text{ex} : s^\text{in} \cup  s^\text{ex} =  \{1, 2, \mydots, d\}\hspace{-2cm} }
p( h[s^\text{in}] = 1, h[s^\text{ex}] = 0 \mid x[s^\text{in}], \zeta ) = 1
\Big).
\end{equation}
For the linear program, this requirement can be expressed through the $A$ matrix and $b$ vector in a straightforward manner.
We set $b = \mathbf{1}$ as a vector of ones with the size of the number of possible values for $x$.
The matrix $A$ gets a first dimension with the same size as $b$ and the second dimension the same size as $\theta$.
Thereby, each row corresponds to a possible value of $x$ and each column to a possible value of $x \odot h$.
Algorithm~\ref{alg:linprog} iterates over each row, representing a possible value of $x$, and then selects each column that corresponds to a possible set of masked features that could occur for $x$ and sets it to one.
As a result, $A\theta = b$ indicates that the probability distribution $\zeta(h \mid x)$ sums to one for each possible value of $x$.

Having constructed $A$, $b$ and $c$, we use SciPy to solve the linear programming problem of Eq.~\ref{eq:appendix:linprogloss}~\cite{2020SciPy-NMeth} and find the optimal value of $\theta$.
Correspondingly, the output of Algorithm~\ref{alg:linprog} is a lookup-table representing the optimal predictor $f^*$, the vector $\theta$ containing the optimal probabilities for $\zeta$, and an index that maps each $(x,h)$ to the element in $\theta$ that contains the corresponding $\zeta(h \mid x)$ value.
If the linear programming solver functions correctly, this solution represents the optimal predictor and selector policies possible for the task.
In our experimental analysis, we assume that the produced solutions are a close approximation of the optimal policies.

\begin{algorithm}[t]
\caption{Our linear programming approach.} 
\label{alg:linprog}
\begin{algorithmic}[1]
\STATE \textbf{Input}: Set of possible features: $X$, Set of possible labels: $Y$, Set of possible masks: $H$,
\\\phantom{\textbf{Input}:}
Probability distribution function: $p(x,y)$, Loss: $L$, Sparsity weight: $\lambda$.
\STATE $\text{feat\_index} \leftarrow \{\}$ \COMMENT{Empty dictionary to map possible masked feature values to indices.}
\STATE $\text{feat\_labels} \leftarrow \{\}$ \COMMENT{Empty dictionary to keep track of label values.}
\STATE $N_\text{unique} \leftarrow 0$ \COMMENT{Counter tracking number of possible unique selected feature values.}
\FOR{$x \in X, y \in Y : p(x,y) > 0$}
\FOR{$h \in H$}
    \IF{$x \odot h \not\in \text{feat\_index}$}
            \STATE $N_\text{unique} \leftarrow N_\text{unique} + 1$
            \STATE $\text{feat\_index}[x \odot h] = N_\text{unique}$ \COMMENT{If unique, the value $x \odot h$ receives the next available index.}
            \STATE $\text{feat\_labels}[x \odot h] = \emptyset$ \COMMENT{Initialize an empty set for every possible associated label value.}
    \ENDIF    
    \STATE $\text{feat\_labels}[x \odot h] \leftarrow \text{feat\_labels}[x \odot h] \cup \{(y, p(x, y))\}$	\COMMENT{Possible labels and cond. probabilities stored per $x\odot h$.}
\ENDFOR
\ENDFOR
\STATE $c \leftarrow \text{zero\_vector}(N_\text{unique})$ \COMMENT{Zero initialization of cost vector of size $N_\text{unique}$.}
\STATE $f^* \leftarrow \{\}$ \COMMENT{Empty dictionary to store optimal predictor.}
\FOR{$x \odot h \in  \text{feat\_index}$}
    \STATE $p(x \odot h) \leftarrow \sum_{(y, p(x, y)) \in \text{feat\_labels}[x \odot h]} p(x,y)$
    \COMMENT{Natural probability of the selected feature values.}
    \STATE $f^*(x \odot h) \leftarrow  \frac{1}{p(x \odot h)}\sum_{(y, p(x, y)) \in \text{feat\_labels}[x \odot h]} p(x,y) y$
    \COMMENT{Assumption: Optimal prediction is the expected value.}
    \STATE $i \leftarrow \text{feat\_index}[x \odot h]$
    \STATE $c[i] \leftarrow  \sum_{(y, p(x, y)) \in \text{feat\_labels}[x \odot h]} p(x,y) \big( L(f^*( x \odot h), y) + \lambda |h| \big)$
\ENDFOR
\STATE $A \leftarrow \text{zero\_matrix}(|X|,N_\text{unique})$ \COMMENT{Zero initialization of constraint matrix of size $|X| \times N_\text{unique}$.}
\STATE $i \leftarrow 0$
\FOR{$x \in X$}
\STATE $i \leftarrow i + 1$
\FOR{$h \in H$}
    \STATE $j \leftarrow \text{feat\_index}[x \odot h]$
    \STATE $A[i, j] \leftarrow 1$
    \COMMENT{Setting ones for every possible selected feature values per row for each $x$.}
\ENDFOR
\ENDFOR
\STATE $b \leftarrow \text{one\_vector}(|X|)$
\COMMENT{Vector of size $|X|$ (number of possible feature values) filled with ones.}
\STATE$\theta \leftarrow \text{Linear\_Program\_Solver}(A, b, c)$  \COMMENT{Solves Eq.~\ref{eq:appendix:linprogloss}, outputs vector of size $N_\text{unique}$ with ordering matching feat\_index.}
\STATE \textbf{Return}: $f^*, \theta, \text{feat\_index}$ %
\end{algorithmic}
\end{algorithm}

\section{Experimental Details}
\label{appendix:experimental-details}

In this section, we describe the experimental details including datasets, implementation, hyper-parameters and additional results.
Our experimental implementation is available here: \url{https://github.com/GarfieldLyu/SUWR}.
For all baselines, we adapted the original source code and when necessary, modified it to fit our experimental objectives. We use the following links for baseline implementation:
\begin{itemize}
    \item L2X: \url{https://github.com/Jianbo-Lab/L2X}
    \item INVASE: \url{https://github.com/jsyoon0823/INVASE}
    \item TabNet: \url{https://github.com/dreamquark-ai/tabnet}
    \item REAL-X: \url{https://github.com/rajesh-lab/realx}
    \item CAE: \url{https://github.com/mfbalin/Concrete-Autoencoders} 
\end{itemize}

All methods are built on neural networks. Among all, L2X, INVASE and REAL-X have independent selector and predictor models. The selector is constructed by feed-forward (FF) layers and outputs a selection probability for each feature. The predictor has a similar architecture but outputs the task-specific prediction, using the selected input by masking out the unselected features. CAE is slightly different, as it uses a single trainable $d\times k$ matrix as the global selector, the matrix values are used as the selection probabilities. The predictor is an FF network, which transforms the selected features (so the input dimension reduces to $k$) and outputs the prediction. TabNet unlike the others, has a single architecture for both selection and prediction. The selection is conducted step-wisely by a neural selection component and the final prediction is generated by ensembling the outputs from all steps.

Our method is flexible in architecture design. We choose to employ a simple model with FF networks to generate selection ($u^t$), prediction ($\hat{y}^{t}$) and stop probability ($p_{\text{stop}}^t$) simultaneously at the step $t$, defined as follows: 
\begin{equation}
    enc = \text{FF}_{\text{enc}}(x\odot h^t), \quad p_{\text{stop}}^t = \text{FF}_{\text{stop}}(enc), \quad u^{t} = \text{FF}_{\text{select}}(enc), \quad \hat{y}^{t} = \text{FF}_{\text{pred}}(enc)
\end{equation} 
$\text{FF}_{\text{enc}}$ is used to encode the input to a hidden representation and across all experiments, we set it to $3$ layers. $\text{FF}_{\text{stop}}$ and $\text{FF}_{\text{select}}$ are both set to $1$ layer. 
$\text{FF}_{\text{pred}}$ is set to $1$ layer for toy and synthetic datasets, and $2$ layers for MNIST datasets. The selection for next step $h^{t+1}$ is sampled from $u^{t}$, and to avoid repeated selection, the probabilities of selected features in corresponding $u^{t}$ are set to $0$ before sampling.

For sparsity-related hyper-parameters, both L2X and CAE require a pre-specified $k$ value as the number of selected features; the rest of methods determine the number of selections by a sparsity weight $\lambda$. Additionally, TabNet also requires a number of steps $n_{steps}$, except for $\lambda$.
For our method, we need to specify a maximum selection budget (or step) $T$, and a sparsity weight $\lambda$ to control the number of selections.
We experimented with a range of values for these hyper-parameters, which we report in the following corresponding subsections.

\subsection{Toy Dataset}
\label{appendix:toy-dataset}
This dataset contains the input of 10-dimensional binary features,  and thus results in $1024 = 2 ^ {10}$ instances. All methods are trained and evaluated on all $1024$ instances. 
For our method, the reported results come from the FF networks with $64$ hidden units, $T=10$, and $\lambda$ in \{0.3, 0.4, 0.5, 0.8\}.
For L2X, we set both selector and predictor as a 3-layer FF model with 64 hidden units, and $k$ from 1 to 10.
For INVASE, we use the same selector and predictor architecture as L2X, and vary $\lambda$ in \{7.0, 8.0, 10.0, 11.0\}.
For TabNet, we vary $n_{steps}$ in \{1, 2, 3\} and $\lambda$ in \{0.002, 0.003, 0.004, 0.005\}.
Lastly for REAL-X, we have to modify the original objective from Cross-Entropy to MSE loss, and vary $\lambda$ in \{1.5, 2.0, 3.0, 5.0\}, the model architecture remains the same as L2X. 
For all methods, we train the model for maximum 2000 epochs, and use early stopping with a patience of 1000 epochs.

\subsection{Synthetic Dataset}
\label{appendix:synthetic-dataset}
For all synthetic datasets (Syn1 -- Syn6), we generate 10,000 training samples with a random seed 0, and 10,000 test samples with a random seed 100. All methods are learned on the training dataset and evaluated on the test dataset, the reported results in Table \ref{tab:syn-results} are averaged over 5 tries.

The selector and predictor model for L2X, INVASE and REAL-X are both 3-layer FF networks with $200$ hidden units. For L2X, we set the $k$ for Syn1 to Syn6 as \{1, 4, 4, 5, 5, 5\}, respectively. For INVASE, we choose $\lambda = 0.1 $ for Syn1 to Syn3, $\lambda = 0.2$ for Syn4 and Syn6, and $\lambda = 0.15 $ for Syn5. For REAL-X, we run the model by varying $\lambda$ across \{ 0.05, 0.1, 0.15, 0.2\} and choose 0.05 for Syn4 and Syn5, 0.1 for Syn6, and 0.15 for Syn1, Syn2 and Syn3.  For TabNet, we chose the recommended hyper-parameters reported in the original paper. In detail, the $n_{steps}$ is set as 4 for Syn1 -- Syn3 and 5 for Syn4 -- Syn6, the $\lambda$ is 0.02 for Syn1, 0.01 for Syn2 and Syn3, and 0.005 for Syn4 -- Syn6. The rest of hyper-parameters in TabNet remain the same as the default setup. 
Our method uses FF networks with $100$ hidden units. For Syn1, we report the results with $T = 4$ and $\lambda = 0.01$. For Syn2 and Syn3, the $T$ is set as 4 and $\lambda$ as 0. For Syn4 and Syn5, we choose $T$ as 5 and $\lambda$ as 0.005. Lastly for Syn6, $T$ is 5 and $\lambda$ is 0.  

We also provide additional results in Figure~\ref{fig:selection_performance} to show the advantages of our method. Firstly, as shown in the left figure, our method is able to select the control-flow feature ($\mathbf{x}_{11}$) at the very first step, as its value determines the upcoming relevant features. We observed that for the other step-wise method TabNet, $\mathbf{x}_{11}$ is usually selected in a later step. Our method in this regard, provides a more interpretable reasoning logic for the selection decision. Furthermore, as the right figure shows, our method has the flexibility to allow us to either explicitly specify a selection budget without sparsity penalty, or figure out the right number of features by tuning a sparsity weight within a maximum selection budget, so that the model can squeeze out irrelevant features and converge to the optimal selection within the budget window.

\begin{figure*}[h]
    \centering
        \includegraphics[width=0.48\textwidth]{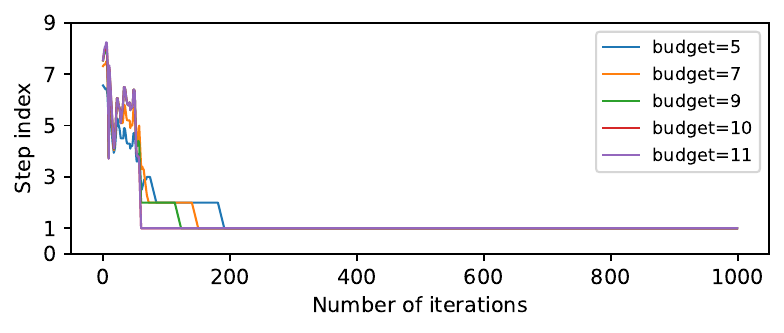}
    \hfill
        \includegraphics[width=0.48\textwidth]{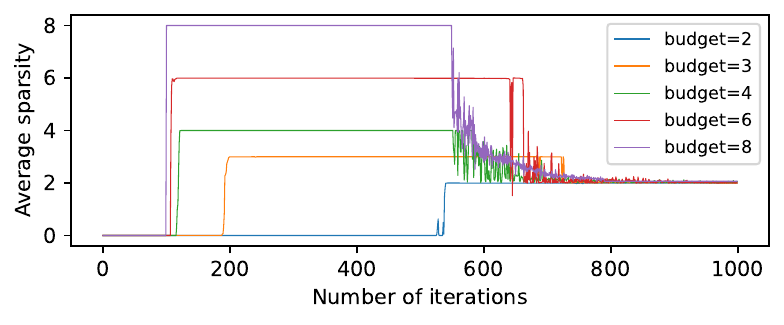}
    \caption{
    \textbf{Left}: at what step $\mathbf{x}_{11}$ is selected for Syn6.
    \textbf{Right}: selection budget vs.\ sparsity for Syn1, where only two features are relevant. }
    \label{fig:selection_performance}
    \end{figure*}

\subsection{MNIST Datasets}
\label{appendix:mnist-datasets}
We follow the benchmark splits for both digits-MNIST and fashion-MNIST. All methods are trained on the 60,000 training samples and tested on the 10,000 test sets. 
The reported results are averaged over 3 tries.

For the predictor without feature selection, we use a 3-layer FF network with 200 hidden units. We choose the same architecture for both selector and predictor for REAL-X. Additionally, we vary the $\lambda$ across $\{1.5, 3.0, 5.0, 6.0, 8.0, 10.0, 12.0, 13.0, 14.0, 18.0, 20.0\}$ to produce the performance curves in figure~\ref{fig:mnist-results}.
For CAE, we run the original code with supervised learning setup. The predictor for CAE is a 3-layer FF with 320 hidden units. We vary $k$ in \{15, 20, 30, 40, 50\} for selection by pixels. For the patch selection, we first obtain the top-k important features learned by CAE, and then train the predictor with the $3\times 3$ patches of features around the selected ones. 
For our method we choose 200 as the FF hidden unit. The maximum selection step $T$ is set to $50$, and the sparsity weight $\lambda$ is varied across $\{0.05, 0.1, 0.15, 0.2, 0.3\}$. 

Additionally, we also include some patch-selection examples from digits-MNIST datasets. Figure~\ref{fig:digit-mnist-steps} again shows the benefits of early stopping in reducing selection while maintaining performance. On the left side, we plot the average number of actual selections (i.e., the average sparsity) can be much smaller than the maximum selection budget $T$, under the same prediction performance. The right side gives a concrete image example of the digit 0. After 4 steps, the model (1) can correctly predict the digit with high confidence; and (2) is recommended to stop here by the stop probability. Continuing the selection will not affect the prediction performance.
Another example in Figure~\ref{fig:mnist_3} shows the process of predicting an image of ``3" with step-wisely selecting patches. The first three patches are enough to distinguish the image from the rest of classes except for ``8", and the fourth path however, shows high discriminative information of ``3" or ``8". This is also supported by the minor perturbation of pixels in the fourth patch. When the fourth patch is not blank anymore, the prediction is flipped from ``3" to ``8".
This example shows a strong example of how the SUWR can explain the contribution of each feature to the prediction, which here gives much more insight than if one would highlight all selected features at once.

    \vspace{-0.5\baselineskip}

\begin{figure*}[hb]
    \begin{center}
        \includegraphics[width=0.45\columnwidth]{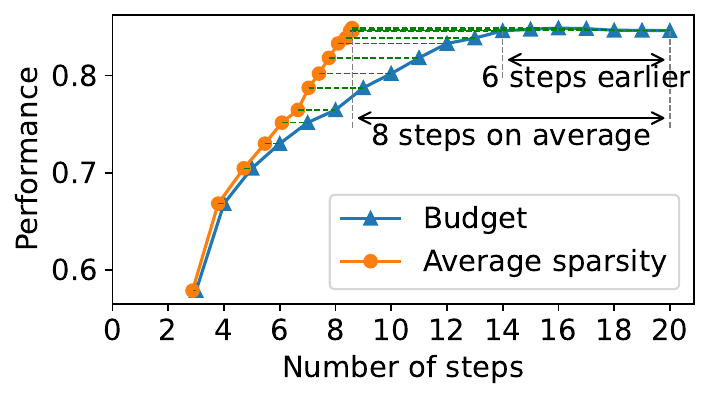}
        \hfill
        \includegraphics[width=0.45\columnwidth]{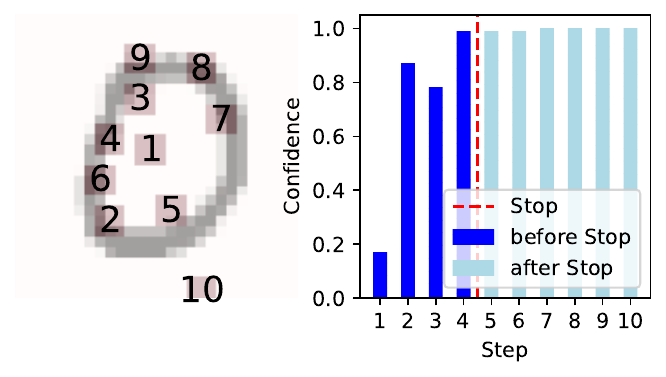}
    \end{center}
    \vspace{-1.5\baselineskip}
    \caption{Selecting by patch on digits-MNIST. Early stopping before maximum step budget $T$.}
    \label{fig:digit-mnist-steps}
    \end{figure*}

    \begin{figure*}[hb]
    \begin{center}
        \includegraphics[width=0.95\textwidth]{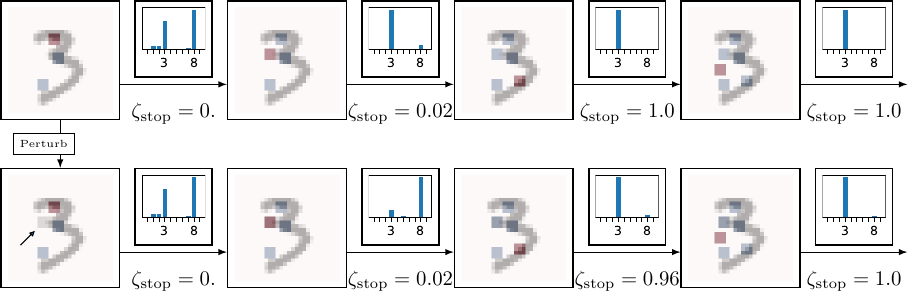}
    \vspace{-0.5\baselineskip}
        \caption{Digit 3 or 8? Each image (indicates one step) is masked by the colored squares and accompanied by a prediction bar chart and a stop probability $\zeta_{\text{stop}}$. The red square is the new selection at the current step and the blue ones indicate the previous selections. The first two steps are omitted. The second row shows the same image as the top row, except for one particular patch which is filled with gray pixels, highlighted by the arrow.
        Due to this change, the prediction on the second row is flipped in the fourth step (second image from the left).}
        \label{fig:mnist_3}
    \end{center}
\end{figure*}

\end{document}